\documentclass[10pt,journal,compsoc]{IEEEtran}

\ifCLASSOPTIONcompsoc
  \usepackage[nocompress]{cite}
  \usepackage{booktabs,makecell,multirow}
  \usepackage[nocompress]{cite}
  \usepackage{tabularx}
  \usepackage{bm}
  \usepackage{amsmath}
  \usepackage{amsthm,amsmath,amssymb}
  \usepackage{mathrsfs}
  \usepackage{lineno,hyperref}
  \usepackage{color}
  \usepackage{graphicx} %use graph format
  \usepackage{grffile}
  \usepackage{subfigure}
  \usepackage{hyperref}

  \usepackage{algorithm}
  \usepackage{algorithmicx}
  \usepackage{algpseudocode}

  \usepackage{bm}

  \usepackage{booktabs}
  \usepackage{makecell}
  \usepackage{amsthm,amssymb,amsfonts}
  \usepackage{diagbox}
  \usepackage[square, comma, sort&compress, numbers]{natbib}
  
  \newtheorem{lemma}{Lemma}
  \usepackage{bbding}
  \usepackage{pifont}
  \usepackage{wasysym}
  \usepackage{amssymb}

  \usepackage{cite}

\else
  \usepackage{cite}
\fi

\ifCLASSINFOpdf

\else

\fi

\begin{document}

\title{Rectified Euler $k$-means and Beyond}

\author{Yunxia~Lin, ~Songcan~Chen
\IEEEcompsocitemizethanks{\IEEEcompsocthanksitem Yunxia~Lin~and~Songcan~Chen are with College of Computer Science and Technology/College of Artificial Intelligence, Nanjing University of Aeronautics and Astronautics, Nanjing, 211106, China and also with MIIT Key Laboratory of Pattern Analysis and Machine Intelligence. Corresponding author is Songcan Chen.
\protect\\ E-mail: \{linyx, s.chen\}@nuaa.edu.cn.}
\thanks{Manuscript received July 24, 2021.}}

\markboth{Journal of \LaTeX\ Class Files,~Vol.~14, No.~8, August~2015}%
{Shell \MakeLowercase{\textit{et al.}}: Euler Clustering}

\IEEEtitleabstractindextext{%
\begin{abstract}
Euler $k$-means (EulerK) first maps data onto the unit hyper-sphere surface of equi-dimensional space via a complex mapping which induces the robust Euler kernel and next employs the popular $k$-means. Consequently, besides enjoying the virtues of $k$-means such as simplicity and scalability to large data sets, EulerK is also robust to noises and outliers. Although so, the centroids captured by EulerK deviate from the unit hyper-sphere surface and thus in strict distributional sense, actually are outliers. This weird phenomenon also occurs in some generic kernel clustering methods. Intuitively, using such outlier-like centroids should not be quite reasonable but it is still seldom attended. To eliminate the deviation, we propose two \textbf{R}ectified \textbf{E}uler {$\bm{k}$}-means methods, i.e., REK1 and REK2, which retain the merits of EulerK while acquire real centroids residing on the mapped space to better characterize the data structures. Specifically, REK1 rectifies EulerK by imposing the constraint on the centroids while REK2 views each centroid as the mapped image from a pre-image in the original space and optimizes these pre-images in Euler kernel induced space. Undoubtedly, our proposed REKs can methodologically be extended to solve problems of such a category. Finally, the experiments validate the effectiveness of REK1 and REK2.
\end{abstract}

}

\maketitle

\IEEEdisplaynontitleabstractindextext
\IEEEpeerreviewmaketitle

\IEEEraisesectionheading{\section{Introduction}\label{sec:introduction}}
\IEEEPARstart{C}{lustering} forms a significant area of research effort {in unsupervised learning} and plays an indispensable role in data mining \cite{data mining1}\cite{data mining2}, pattern recognition, machine learning \cite{clusteringApp_machine learning1}\cite{clusteringApp_machine learning2} and image processing \cite{image processing}. It aims to partition objects into several clusters so that the objects in the same cluster are highly similar, whereas the data in different clusters are significantly different over certain similarity measurements \cite{what is clustering}.

So far, many clustering algorithms have been proposed and mainly cover several types: density-based method, hierarchical method, graph theoretic method, objective function based method, large margin clustering method \cite{clusteringApp_machine learning1}\cite{clustering category 1}-\cite{clustering category 3}. Among such large volume of clustering methods, $k$-means \cite{kmeans} belonging to the objective function based has been studied extensively \cite{kmeans-wide use} and is still one of the most popular clustering algorithms \cite{kmeans-popular}. $k$-means aims to partition the data points into a predefined number of clusters by minimizing the sum of square of Euclidean distance between the samples and the centroids. It enjoys simplicity, efficiency and low computational complexity $\mathcal{O}(Tndk)$ \cite{kmeansCom} where $n$ is the amount of samples, $T$ is the number of iterations, $k$ is the number of clusters. Besides it can be easily implemented by just a few lines of codes. Actually, the Euclidean measure in the objective function of $k$-means implicitly assumes that the data lie on isolated elliptical regions \cite{KmeansAssume}, thus it cannot partition clusters with complicated non-linear structures in input space. To solve this issue, kernelised versions of $k$-means as its most important extension are successively proposed \cite{kernel kmeans 1}-\cite{kernel kmeans 4} and have been widely studied due to its promising merits \cite{kernel clustering 1}\cite{kernel clustering 2}.
%In these kernelized counterparts proposed so far, Euler $k$-means \cite{EulerKmeans1}\cite{EulerKmeans2} \textcolor{red}{stands out for its simplicity, elegance, uniqueness and ease of implementation}.

The generic kernel $k$-means methods first perform a non-linear transformation implicitly from the original space to a much higher or even infinite dimensional feature space (i.e., Reproducing Kernel Hilbert Space (RKHS)) where the data are more likely separable and next perform $k$-means on these mapped data, yielding so-called kernelized $k$-means. Because there are amount of different kernels, various kernel $k$-means methods are successively proposed, in which the well-known and commonly used ones are Radial Basis Function kernel (RBF) based $k$-means \cite{RBF kmeans} and spatial pyramid matching kernel based $k$-means \cite{SPM kmeans}. Such generic version indeed performs well on nonlinear scenarios, but it also sacrifices the simplicity and low computational complexity of the $k$-means due to needing to calculate and save the kernel matrix of size $n\times n$ whose cost is quadratic to the size of the dataset, consequently making it unsuitable to large-scale corpora. To solve this problem, some explicit-mapping-based kernel methods \cite{App1}-\cite{App5} have been proposed, which resort to various random feature mappings to approximate some specific nonlinear kernels involved and enjoy a linear time complexity in both the data size $n$ and the mapped dimension $D$. However, these methods either sacrifice some discriminative information or also suffer from a relatively high computational complexity when $D$ is large. Especially, they have to tune manually an important mapping dimension to guarantee their performances. Compared with the above two implicit and explicit categories, Euler $k$-means (EulerK) recently proposed stands out for its simplicity, elegance, uniqueness and ease of implementation\cite{EulerKmeans1}\cite{EulerKmeans2}. It marries their merits, i.e., utilizing an explicit equi-dimensional mapping corresponding to an exact kernel and solves their troubles. Specifically, it first maps $d$-dimensional data point $x$ onto an equi-dimensional RKHS space via an explicit complex transformation $f(x)=\frac{1}{\sqrt{2}}e^{i\alpha \pi x}$ where $\alpha$ is a hyperparameter and $i=\sqrt{-1}$ which {\emph{exactly}} induces the cosine-metric-based Mercer kernel, i.e., Euler kernel, and next performs the popular $k$-means on the mapped data, thus masterly avoiding of the calculation of large kernel matrix. Consequently, EulerK solves the nonlinear clustering problem that $k$-means cannot deal with well while perfectly inherits the virtues of $k$-means including simplicity, scalability to large datasets, low computational complexity linear in the data size and easy implementation. Moreover, EulerK yet enjoys robustness \cite{EulerKmeans1}\cite{EulerKmeans2} that both $k$-means and traditional kernel $k$-means do not have. Furthermore, the methodology to extend the traditional $k$-means can be naturally operated on EulerK to enhance and further improve the EulerK.

Albeit carrying on so many merits, as Figure \ref{Halfmoon1DEK} and \ref{Halfmoon2DEK} show, EulerK as a centroid-based kernel $k$-means actually acquires the centroids which deviate from the mapped space. In strict distributional sense, these centroids are actually outliers \cite{OutlierCite}. This {weird} phenomenon also {likewise} occurs in existing typical kernel clustering methods such as kernel $k$-means \cite{kernel kmeans 1}, kernel Fuzzy c-means (KFCM) \cite{KernelFCM}, kernel Self Organizing Map (KSOM) \cite{KSOM1}\cite{KSOM2},  Kernel Neural Gas \cite{KernelNeuralgas} based on Radial Basis Function kernel and Laplace kernel, Multiple Kernel $k$-means (MKKM) \cite{MultiKernelKmeans} and Kernelized Mahalanobis Distance based Fuzzy Clustering (KMD-FC) \cite{KMD-FC}. Obviously, using such outlier-like centroids to represent the data is not quite reasonable but {is indeed seldom concerned} so far.
%radial basis kernel based ones \cite{KernelKmeans}\cite{EulerKernel}\cite{Event}.
\par{In this work, in order to demonstrate and analyse this phenomenon, we first define a criterion named \emph{Deviation Degree} $\kappa$ based on the distance between these centroids obtained by EulerK to the unit hyper-sphere surface the data reside on. Then motivated by this observation, we try to rectify the deviation and develop two rectified EulerKs, namely Rectified Euler $k$-means 1 (REK1) and Rectified Euler $k$-means 2 (REK2) respectively, which not only inherit the merits of EulerK but also acquire the centroids exactly residing on the support domain of data distribution. Specifically, REK1 builds the constraint on the centroids in the complex space, thus rectifying EulerK to assure the strict belonging of the obtained centroids to the mapped space. While REK2 views each centroid as the mapped image from a corresponding data point or pre-image in the original space and optimizes these pre-images in Euler kernel induced distance metric space. In summary, the contributions of this work can be summarized as follows:}
\begin{itemize}
\item We find that EulerK actually acquires outlier-like centroids which are possibly unsuitable to represent the distribution of the mapped data and define a criterion named \emph{Deviation Degree} $\kappa$ to quantitatively measure the deviation degree from these centroids to the support domain. {This weird phenomenon also occurs in some specific kernels beyond Euler kernel based clustering methods, such as Radial Basis Function kernel or Laplace kernel.}
\item We introduce two rectified methods named REK1 and REK2 to rectify the deviation which retain the desirable merits of EulerK while acquire the centroids residing on the support domain to better represent the data structures. Moreover, our proposed REK1 and REK2 can methodologically be extended straightforwardly to deal with problems of such a category {beyond Euler $k$-means}.
\item We carry out experiments on a synthetic dataset and several commonly used real datasets, and the experimental results validate the rationality and effectiveness of our proposed REK1 and REK2.
\item We also find an interesting phenomenon that the insensitive range to parameter $\alpha$ of EulerK on high-dimensional datasets is exactly opposite of that on low-dimensional ones, however this is not mentioned in EulerK. Moreover, the similar phenomenon also occurs in both REK1 and REK2.
\end{itemize}
\par{We compare REK1 and REK2 with EulerK in Table \ref{DisCen} which shows their formal expressions of distance metric and centroid, respectively.}

In the rest of this paper, Section 2 briefly overviews the preliminaries including $k$-means, kernel $k$-means, Euler kernel and Euler $k$-means. Section 3 details our proposed REK1 and REK2. Section 4 reports extensive experimental results and analysis. Finally, Section 5 concludes this paper with future research directions.

\begin{table*}[]
\caption{Distance and centroid of EulerK, REK1 and REK2.}
\centering
\begin{tabular}{c|c|c}

\hline
method      & distance between a mapped data point $f(x_j)$ and a centroid $\mathbf{m}_c=\mathbf{a}_c+i\mathbf{b}_c$ & centroid $\mathbf{m}_c$\\
\hline
EulerK      & $\frac{d}{2}+||\mathbf{m}_c||^2-\cos(\alpha \pi x_j)^T\mathbf{a}_c-\sin(\alpha \pi x_j)^T\mathbf{b}_c$        & $\sum_{j \in {\Omega_{c}}}{\phi(\mathbf{x}_{j,l})}$        \\
\hline
REK1        & $\frac{d}{2}+||\mathbf{m}_c||^2-\cos(\alpha \pi x_j)^T\mathbf{a}_c-\sin(\alpha \pi x_j)^T\mathbf{b}_c$         & $\frac{{\sum_{j \in {\Omega_{c}}}}{\phi(\mathbf{x}_{j,l})}}{\sqrt{(\sum_{{j \in {\Omega_{c}}}}\phi(\mathbf{x}^{i1}_{j,l}))^2+(\sum_{{j \in {\Omega_{c}}}}\phi(\mathbf{x}^{i2}_{j,l}))^2}}$        \\
\hline
REK2        & $\sum_{l=1}^{d}(1-\cos(\alpha\phi x_{j,l}-\mathbf{u}_{c,l}))$        & -$ \arccos(\frac{{\sum_{j \in {\Omega_{c}}}}{\cos(\alpha \pi x_{j,l})}}{A})$ (pre-image of $\mathbf{m}_c$)\\
\hline
\end{tabular}
\label{DisCen}
\begin{itemize}
\item $A = \sqrt{(\sum_{j \in \Omega_c}{\cos(\alpha\pi x_{j,l})})^2+(\sum_{j \in \Omega_c}{\sin(\alpha \pi x_{j,l})})^2}$.
%\item -$ \arccos(\frac{{\sum_{j \in {\Omega_{c}}}}{\cos(\alpha \pi x_{j,l})}}{A})$ is the pre-image of the centroid $\mathbf{m}_c$.
\end{itemize}
\end{table*}

\section{PRELIMINARIES}
Before developing our proposed Rectified Euler $k$-means 1 and Rectified Euler $k$-means 2, we first briefly introduce preliminaries including $k$-means \cite{kmeans}, kernel $k$-means \cite{kernel kmeans 1}, Euler kernel \cite{EulerKernel} and Euler $k$-means \cite{EulerKmeans1}\cite{EulerKmeans2}.

\subsection{$k$-means}
Given a data set $\mathbf{X}=\{x_{1},\cdots,x_{n}\}$ where $x_{i}\in \mathcal{R}^d$ and $n$ is the number of samples, the objective of $k$-means is to minimize the sum-of-square loss over the cluster assignment matrix $\mathbf{P}\in \{0,1\}^{n\times k}$ and the cluster centroid matrix $\mathbf{M}\in \mathcal{R}^{n \times k}$, which can be formulated as the following optimization problem,
\begin{equation}
\begin{aligned}
  & \underset{\mathbf{P},\mathbf{M}}{\text{min}}
  & &\sum_{c=1}^{k}\sum_{j=1}^{n}\mathbf{P}_{jc}||x_{j}-\mathbf{m}_{c}||^2, \\
  & \text{s.t.}
  & & \mathbf{P}_{jc} \in \{0,1\}; \quad \sum_{c=1}^{k} \mathbf{P}_{jc}=1.
\end{aligned}
\label{kmeans}
\end{equation}
Here, $k$ is a predefined number of clusters and $\mathbf{m}_{c}$ is the centroid which is updated by the mean vector of data points in the $c$-th ($1\leq c \leq k$) group, i.e.,
\begin{equation}
\mathbf{m}_{c} = \frac{\sum_{j\in\Omega_{c}}x_{j}}{\sum_{j\in\Omega_{c}}\mathbf{P}_{jc}}.
\label{kmeans_M}
\end{equation}
Note $\Omega_{c}$ is a set of subscripts of samples belonging to the $c$-th partition. After obtaining the centroids, the cluster membership indicator matrix $\mathbf{P}$ is given as
\begin{equation} \mathbf{P}_{jc}=\left\{
\begin{aligned}
1, &  & if \quad  c = arg min \quad||\mathbf{x}_j-\mathbf{m}_{c}||^2; \\
0, &  & otherwise.
\end{aligned}
\right.
\label{kmeans_P}
\end{equation}
We can get the local optimum of (\ref{kmeans}) via iterative update between (\ref{kmeans_M}) and (\ref{kmeans_P}).
\par{$k$-means is simple, efficient and enjoys low computational complexity linear on the dataset size, but its corresponding objective function (\ref{kmeans}) actually assumes that the data lie on isolated elliptical regions \cite{KmeansAssume}, thus it is unsuitable to separate non-linearly separable data in the original space, limiting its popularity and feasibility in many practical applications.}

\subsection{Kernel $k$-means}
% What is kernel kmeans? what is the merit and weakness of kernel kmeans?
The generic kernel $k$-means methods firstly map data points in the original space into a much higher even infinite dimensional Reproducing Kernel Hilbert Space (RKHS) via an implicit mapping induced by a kernel function. They next cluster the mapped data points in such RKHS to implement the assignment such that the similar data points are in the same cluster and dissimilar data points are in different groups \cite{kernel kmeans 1}.

\par{Specifically, let $\phi(\cdot):x\in \mathcal{R}^d \longrightarrow \mathcal{H}$} be a feature mapping induced by a kernel function which maps $x$ onto a reproducing kernel hilbert space $\mathcal{H}$, the objective of kernel $k$-means can be formulated as the following optimization problem,
\begin{equation}
\begin{aligned}
  & \underset{\mathbf{P},\mathbf{M}}{\text{min}}
  & &\sum_{c=1}^{k}\sum_{j=1}^{n}\mathbf{P}_{jc}||\phi(x_{j})-\mathbf{m}_{c}||^2, \\
  & \text{s.t.}
  & & \mathbf{P}_{jc} \in \{0,1\}; \quad \sum_{c=1}^{k} \mathbf{P}_{jc}=1,
\end{aligned}
\label{KernelKmeans}
\end{equation}
where $\mathbf{m}_{c} = \frac{\sum_{j\in\Omega_{c}}\phi(x_{j})}{\sum_{j\in\Omega_{c}}\mathbf{P}_{jc}}$.

Because $\phi(\cdot)$ is usually an implicit mapping, we can utilize kernel trick to formulate (\ref{KernelKmeans}) as the following matrix form
\begin{equation}
\begin{aligned}
  & \underset{\mathbf{P}}{\text{min}}
  & &\text{Tr}(\mathbf{K})-\text{Tr}(\mathbf{L}^\frac{1}{2}\mathbf{P}^T\mathbf{K}\mathbf{P}\mathbf{L}^{\frac{1}{2}}) \\
  & \text{s.t.}
  & & \mathbf{P}\in \{0,1\}^{n\times k}; \quad \mathbf{P}\mathbf{1}_{k}=\mathbf{1}_{n}.
\end{aligned}
\label{KernelKmeansMatrix}
\end{equation}
where $\mathbf{K}$ is a kernel matrix with $\mathbf{K}_{ij}=\phi(x_i)^{T}\phi(x_j)$, $\mathbf{L}=diag([n_{1}^{-1},\cdots,n_{k}^{-1}])$ with $n_{c}={\sum_{j\in\Omega_{c}}\mathbf{P}_{jc}}$ and $\mathbf{1}_{l}\in \mathcal{R}^{l}$ is a column vector with all elements 1.

While kernel $k$-means performs well on nonlinear samples, it is unsuitable for large-scale clustering problems due to needing to compute and save a kernel matrix of size $n \times n$, whose cost is quadratic to the size of the dataset. Moreover, we find that the centroids computed in kernel $k$-means based on some certain kernels, such as Radial Basis Function (RBF) kernel \cite{RBF kmeans} and Laplace kernel \cite{Laplace kernel} actually deviate from the support domain of the mapped data, which can be viewed as outliers in strictly distributional sense \cite{OutlierCite}.

\subsection{Euler Kernel}
% 突出Euler kernel与普通核的不同之处，且欧拉核所带来的advantage
% 介绍清楚什么是欧拉核
The ($j,q$)-entry of Euler kernel matrix is shown as follows \cite{EulerKernel}
\begin{equation}
\begin{aligned}
   \mathbf{K}_{j,q}
   = \quad &\frac{1}{2}\sum_{l=1}^{d}\cos(\alpha\pi(x_{j}(l)-x_{q}(l))) \\
  %& \text{s.t.}
   & -i\frac{1}{2}\sum_{l=1}^{d}\sin(\alpha\pi(x_{j}(l)-x_{q}(l))),
\end{aligned}
\end{equation}
where $i$ is the imaginary unit. We can find that different from existing traditional Mercer kernels, Euler kernel is defined on the complex space.
\par{In fact, \cite{EulerKernel} has shown that Euler kernel induces an explicit complex mapping $\phi(x)$ that maps a $d$-dimensional real data point $x_j$ onto a equi-dimensional RKHS space, just as follows}
\begin{equation}
\phi(x_j)=\frac{1}{\sqrt{2}}e^{i\alpha\pi x_j}=\frac{1}{\sqrt{2}}(\cos(\alpha\pi x_j)+i\sin(\alpha\pi x_j)).
\end{equation}
Naturally, the square Euclidean distance between $\phi(x_j)$ and $\phi(x_q)$ in the complex RHKS space is given by
\begin{equation}
\begin{aligned}
d(\phi(x_j),\phi(x_q))&=||\phi(x_j)-\phi(x_q)||^2\\&=\sum_{l=1}^{d}(1-\cos(\theta_{j}(l)-\theta_{q}(l))).
\end{aligned}
\end{equation}
Although the mapping $\phi(x)$ and Euler kernel are defined on the complex space, $d(\cdot,\cdot)$ is a real value so it can be used to measure the dissimilarities among data points.
\par{Very different from the traditional kernels, Euler kernel induces an explicit and equi-dimensional complex mapping and thus masterly avoids the high computational complexity and memory cost. Moreover, this explicit mapping is bounded on [-1,1], meaning that the Euler kernel is robust to noise and outliers.}

\subsection{Euler $k$-means}
Euler $k$-means \cite{EulerKmeans1}\cite{EulerKmeans2} firstly maps $d$-dimensional data points onto the equi-dimensional RKHS space via a complex mapping $f(x)=\frac{1}{\sqrt{2}}e^{i\alpha \pi x}$ and then employs the popular $k$-means on the mapped data.
\par{Specifically, the representative centroid $\mathbf{m}_{c}$ of the $c$-th group can be explicitly represented in the complex space as}
\begin{equation}
\mathbf{m}_{c}=\frac{1}{\sqrt{2}}(\mathbf{a}_{c}+i\mathbf{b}_{c}),
\label{EulerKmeansCentre}
\end{equation}
where $\mathbf{a}_{c}=\frac{1}{n_c}\sum_{j\in \Omega_{c}}\cos(\alpha \pi x_j)$ and $\mathbf{b}_{c}=\frac{1}{n_c}\sum_{j\in \Omega_{c}}\sin(\alpha \pi x_j)$, $\Omega_{c}$ is a mark set of data points belonging to the $c$-th partition.
\par{Based on the explicit representation of $\mathbf{m}_c$, we can derive the criterion of Euler $k$-means as follows}
\begin{equation}
\begin{aligned}
  & \underset{\mathbf{P},\mathbf{M}}{\text{min}}
  & &\sum_{c=1}^{k}\sum_{j=1}^{n}\mathbf{P}_{jc}||\phi(x_{j})-\mathbf{m}_{c}||^2, \\
  & \text{s.t.}
  & & \mathbf{P}_{jc} \in \{0,1\}; \quad \sum_{c=1}^{k} \mathbf{P}_{jc}=1.
\end{aligned}
\label{EulerKmeans}
\end{equation}
where $\mathbf{m}_c$ is given by (\ref{EulerKmeansCentre}), and $||\phi(x_{j})-\mathbf{m}_{c}||^2$ is given by
\begin{equation}\label{EulerKmeansDis}
\frac{d}{2}+||\mathbf{m}_c||^2-\cos(\alpha \pi x_j)^T\mathbf{a}_c-\sin(\alpha \pi x_j)^T\mathbf{b}_c.
\end{equation}
Euler $k$-means optimizes (\ref{EulerKmeans}) via an iterative update strategy until it converges at its local optimum: (1) explicitly updates the representative prototype according to (\ref{EulerKmeansCentre}) by fixing the label matrix; (2) assigns the data points to their closest centroid according to (\ref{EulerKmeansDis}) by fixing the centroids.

{It is easy to conclude that in addition to inheriting the merits of $k$-means of simplicity and scalability to large data sets, Euler $k$-means is also robust to noise and outliers. However, we find that Euler $k$-means determines the centroids which deviate from the support domain of the mapped data, these centroids are actually outliers in strictly distributional sense \cite{OutlierCite}. It should not be quite reasonable to utilize these outlier-like centroids to represent the distribution of data points.}

\section{Rectified Euler $k$-means}
In this section, we detail our proposed Rectified Euler $k$-means 1 (REK1) and Rectified Euler $k$-means 2 (REK2), respectively. In the following, we provide their model formulation, optimization strategy as well as convergence analysis in separated sub-sections.
\subsection{Rectified Euler $k$-means 1}
\subsubsection{Model Formulation}
Following Euler $k$-means (EulerK), we also model the $c$-th centroid as $\mathbf{m}_c=\frac{1}{\sqrt{2}}(\mathbf{a}_c+i\mathbf{b}_c)$. But in order to guarantee that the centroids exactly reside on the support domain of the mapped data, we rectify EulerK by imposing the constraint of unit length on each centroid, thus we formulate REK1 as
  \begin{equation}
  \begin{aligned}
  & \underset{\mathbf{P},\mathbf{M}}{\text{min}}
  & &\sum_{c=1}^{k}\sum_{j=1}^{n}\mathbf{P}_{jc}||\phi(\mathbf{x}_{j})-\mathbf{m}_{c}||_{2}^2, \\
  & \text{s.t.}
  & &\mathbf{a}_{c,l}^2+\mathbf{b}_{c,l}^2=1, \quad c = 1,\cdots,k; l = 1,2,\cdots,d, \\
  && &\mathbf{P}_{jc} \in \{0,1\}; \quad \sum_{c=1}^{k} \mathbf{P}_{jc}=1,
  \end{aligned}
  \label{REK1}
  \end{equation}
where $\mathbf{a}_{c,l}$ is the $l$-th dimension of $\mathbf{a}_c$.
\subsubsection{Problem Solution}
The objective function (\ref{REK1}) of REK1 is an equality constrained optimization problem, its Lagrangian formula is as follows:
  \begin{equation}
  \begin{aligned}
  %\underset{\mathbf{P},\mathbf{M}}{\text{min}}
  f(\mathbf{P},\mathbf{M},\mathbf{\Lambda})=
  & \frac{1}{2}\sum_{c=1}^{k}\sum_{j=1}^{n}\mathbf{P}_{jc}||\phi(\mathbf{x}_{j})-\mathbf{m}_{c}||_{2}^2 \\
  & + \frac{1}{2}\sum_{c=1}^{k}\sum_{l=1}^{d}\mathbf{\Lambda}_{c,l}(\mathbf{a}_{c,l}^2+\mathbf{b}_{c,l}^2-1),
%  & \text{s.t.}
%  && \mathbf{a}_{c,i}^2+\mathbf{b}_{c,i}^2=1, i \in {1,2,\cdots,d}; c \in {1,\cdots,k}. \\
  \end{aligned}
  \label{REK1Lag}
  \end{equation}
where $\mathbf{\Lambda} \in \mathcal{R}^{d\times k}$ is dual variable matrix. We resort to the iterative update strategy to acquire the final solution.
\par{Fix $\mathbf{P}$ and $\mathbf{\Lambda}$, the derivative of (\ref{REK1Lag}) with respect to $\mathbf{m}_{c,l}$ is}
\begin{equation}
\frac{\partial f}{\partial \mathbf{m}_{c,l}} = -\sum_{j=1}^{n}\mathbf{P}_{jc}(\phi(\mathbf{x}_{j,l})-\mathbf{m}_{c,l})+\mathbf{\Lambda}_{c,l}\mathbf{m}_{c,l}.
\end{equation}
where $\mathbf{m}_{c,l}$ is the $l$-th dimension of $\mathbf{m}_c$ and $\mathbf{\Lambda}_{c,l}$ is the value of $c$-th row and $l$-th column of $\mathbf{\Lambda}$.
\par{Fix $\mathbf{M}$ and $\mathbf{\Lambda}$, the derivative of (\ref{REK1Lag}) with respect to} $\mathbf{P}_{jc}$ is
\begin{equation}
\frac{\partial f}{\partial \mathbf{P}_{jc}} = ||\phi(\mathbf{x}_{j})-\mathbf{m}_{c}||_{2}^2.
\end{equation}
\par{Fix $\mathbf{P}$ and $\mathbf{M}$, the derivative of (\ref{REK1Lag}) with respect to $\mathbf{\Lambda}_{c,l}$ is}
\begin{equation}
\frac{\partial f}{\partial \mathbf{\Lambda}_{c,l}} = \mathbf{a}_{c,l}^2+\mathbf{b}_{c,l}^2-1.
\end{equation}
Let $\frac{\partial f}{\partial \mathbf{m}_{c,l}}=0$, $\frac{\partial f}{\partial \mathbf{P}_{jc}}=0$, $\frac{\partial f}{\partial \mathbf{\Lambda}_{c,l}}$ = 0, we can obtain that $\mathbf{m}_{c,i}$ = $\frac{{\sum_{j \in {\Omega_{c}}}}{\phi(\mathbf{x}_j)}}{n_{c}+\mathbf{\Lambda}_{jc}}$ and $n_{c}+\mathbf{\Lambda}_{jc}=\sqrt{(\sum_{{j \in {\Omega_{c}}}}\cos(\mathbf{x}_{j,l}))^2+(\sum_{{j \in {\Omega_{c}}}}\sin(\mathbf{x}_{j,l}))^2}$.
%, where $\mathbf{x}^{i1}$ and $\mathbf{x}^{i2}$ are the real part and imaginary part of $\mathbf{x}$, respectively.
Therefore, the update formula of $\mathbf{m}_{c,l}$ is as follows
  \begin{equation}
  \mathbf{m}_{c,l} = \frac{{\sum_{j \in {\Omega_{c}}}}{\phi(\mathbf{x}_{j,l})}}{\sqrt{(\sum_{{j \in {\Omega_{c}}}}\cos(\mathbf{x}_{j,l}))^2+(\sum_{{j \in {\Omega_{c}}}}\sin(\mathbf{x}_{j,l}))^2}}.
  \label{REK1UpdateM}
  \end{equation}
the update formula of $\mathbf{P}_{jc}$ is as follows
\begin{equation} \mathbf{P}_{jc}=\left\{
\begin{aligned}
1, &  & if \quad  c = arg min \quad ||\phi(\mathbf{x}_j)-\mathbf{m}_{c}||^2; \\
0, &  & otherwise.
\label{REK1UpdateP}
\end{aligned}
\right.
\end{equation}
In summary, REK1 optimizes the criterion (\ref{REK1}) via the follwoing iterative update strategy until (\ref{REK1}) converges:
\begin{itemize}
  \item update $\mathbf{m}_c$ according to (\ref{REK1UpdateM}) by fixing $\mathbf{P}$;
  \item assign the label for each data point according to (\ref{REK1UpdateP}) by fixing $\mathbf{M}$.
\end{itemize}
\subsubsection{Convergence Analysis and Complexity Analysis}
It is obvious that the objective function (\ref{REK1}) of REK1 is bounded. The objective function (\ref{REK1}) will monotonically decrease via iterative update for the cluster centroids and assignment labels for the data points. Therefore, following \textbf{Theorem 11} in \cite{ConvergenceKmeans}, the REK1 algorithm will converge to a local minimum in a finite number of iterations.
\par{Note that REK1 inherits the low computational complexity of EK, which is comparable to $k$-means. Specifically, REK1 takes $\mathcal{O}(nd)$ and $\mathcal{O}(ndk)$ for updating the centroid matrix $\mathbf{M}$ and label assignment matrix $\mathbf{P}$ in each iteration, respectively. In summary, REK1 takes $\mathcal{O}(Tndk)$ until convergence, where $T$ is the number of iterations. As Figure \ref{ConvergenceREK1_Scene13} and \ref{ConvergenceREK1_Caltech256} show, $T$ is very small, so the REK1 algorithm enjoys low computational complexity while it is a nonlinear clustering approach.

\subsection{Rectified Euler Kmeans 2}
\subsubsection{Model Formulation}
In order to assure the strict belonging to the support domain of mapped data for the centroids, we view each centroid as a mapped image of one point in the original space, i.e.,
$\mathbf{m}_c=\frac{1}{\sqrt{2}}e^{i\mathbf{u}_{c}}=\frac{1}{\sqrt{2}}(\cos(\mathbf{u}_{c})+i\sin(\mathbf{u}_{c}))$, so the square Euclidean distance $||\phi({\mathbf{x}_{j}})-\mathbf{m}_{c}||_{2}^{2}$ between $\phi({\mathbf{x}_{j}})$ and $\mathbf{m}_{c}$ is given by
\begin{equation}\label{REK2Dis}
  ||\phi({\mathbf{x}_{j}})-\mathbf{m}_{c}||_{2}^{2} = \sum_{l=1}^{d}(1-\cos(\mathbf{\theta}_{j,l}-\mathbf{u}_{c,l})),
\end{equation}
where $\theta_{j,l}=\alpha\pi x_{j,l}$ and $x_{j,l}$ is the $l$-th dimension of $x_j$.
Therefore, REK2 is formulated as follows:
%\begin{equation}
%\begin{aligned}
%  & \underset{\mathbf{U},\mathbf{P}}{\text{min}}
%  & &\sum_{c=1}^{k}\sum_{j=1}^{n}\mathbf{P}_{jc}\sum_{i=1}^{d}(1-\cos(\mathbf{\theta}_{j,l}-\mathbf{u}_{c,l})),  \\
%  & \text{s.t.}
%  & & \mathbf{P}_{jc} \in \{0,1\}; \quad \sum_{c=1}^{k} \mathbf{P}_{jc}=1.
%\end{aligned}
%\label{kmeans}
%\end{equation}

  \begin{equation}
  \begin{aligned}
   &\underset{\mathbf{U},\mathbf{P}}{\text{min}}
   &&\sum_{c=1}^{k}\sum_{j=1}^{n}\mathbf{P}_{jc}\sum_{i=1}^{d}(1-\cos(\mathbf{\theta}_{j,l}-\mathbf{u}_{c,l})), \\
   & \text{s.t.}
   && \mathbf{P}_{jc} \in \{0,1\}; \quad \sum_{c=1}^{k} \mathbf{P}_{jc}=1, \\
  \end{aligned}
  \label{REK2}
  \end{equation}
where $\mathbf{U}\in \mathcal{R}^{k\times d}$ is the preimage of the centre matrix in the mapped space.
\subsubsection{Problem Solution}
We optimize (\ref{REK2}) by the iterative update strategy until (\ref{REK2}) converges.
\par{\textbf{Step 1: Fix $\mathbf{P}$, update $\mathbf{U}$}}

If we fix $\mathbf{P}$, (\ref{REK2}) with respect to $\mathbf{U}$ becomes
  \begin{equation}
   \underset{\mathbf{U}}{\text{min}}
   \sum_{c=1}^{k}\sum_{j=1}^{n}\mathbf{P}_{jc}
   \sum_{l=1}^{d}(1-\cos(\mathbf{\theta}_{j,l}-\mathbf{u}_{c,l})),
  \label{ConvSS_1}
  \end{equation}
The update of each dimension of $\mathbf{u}_c$ is uncoupled, so we can update $\mathbf{u}_c$ dimension by dimension as follows:
  \begin{equation}
  %\begin{aligned}
   \underset{\mathbf{u}_{c,l}}{\text{min}}
   \sum_{\mathbf{x}_{j}\in \mathbf{\Omega}_{c}}
   (1-\cos(\mathbf{\theta}_{j,l}-\mathbf{u}_{c,l})),
  %\end{aligned}
  \label{ConvSS_1}
  \end{equation}
That is to optimize the following formula
  \begin{equation}
  %\begin{aligned}
   \underset{\mathbf{u}_{c,l}}{\text{max}}
   \sum_{\mathbf{x}_{j}\in \mathbf{\Omega}_{c}}
   \cos(\mathbf{\theta}_{j,l}-\mathbf{u}_{c,l}),
  %\end{aligned}
  \label{REK2Centre1}
  \end{equation}
Unfolding (\ref{REK2Centre1}), we obtain
  \begin{equation}
  %\begin{aligned}
   \underset{\mathbf{u}_{c,l}}{\text{max}}
   \sum_{\mathbf{x}_{j}\in \mathcal{G}_{c}}
   \cos(\mathbf{\theta}_{j,l})\cos(\mathbf{u}_{c,l})+\sin(\mathbf{\theta}_{j,l})\sin(\mathbf{u}_{c,l}),
  %\end{aligned}
  \label{REK2Centre2}
  \end{equation}
(\ref{REK2Centre2}) next can be reformulated as
  \begin{equation}
  %\begin{aligned}
   \underset{\mathbf{u}_{c,l}}{\text{max}}
   %\sum_{\mathbf{x}_{j}\in \mathcal{G}_{c}}
   \quad A\cos(B+\mathbf{u}_{c,l}),
  %\end{aligned}
  \label{REK2Centre3}
  \end{equation}
where $A = \sqrt{(\sum_{j \in \Omega_c}{\cos(\alpha\pi x_{j,l})})^2+(\sum_{j \in \Omega_c}{\sin(\alpha \pi x_{j,l})})^2}$ and $B = \arccos(\frac{{\sum_{j \in {\Omega_{c}}}}{\cos(\alpha \pi x_{j,l})}}{A})$. Finally, we obtain the update equation of $\mathbf{u}_{c,l}$ as $\mathbf{u}_{c,l}=-B$.
\par{\textbf{Step 2: Fix $\mathbf{U}$, update $\mathbf{P}$}}
\begin{equation} \mathbf{P}_{jc}=\left\{
\begin{aligned}
1, &  & if \quad  c = arg min \quad ||\phi(\mathbf{x}_j)-\mathbf{m}_{c}||^2; \\
0, &  & otherwise.
\label{REK2UpdateP}
\end{aligned}
\right.
\end{equation}
where $\mathbf{m}_c=\phi(\mathbf{u}_c)$.

\subsubsection{Convergence Analysis and Complexity Analysis}
Obviously, the objective function (\ref{REK2}) of REK2 is also bounded. Moreover, (\ref{REK2}) monotonously decreases in the iterative process for updating the centroid matrix $\mathbf{M}$ and the cluster assignment matrix $\mathbf{P}$. Therefore, the REK2 algorithm will converge at a local minimum after a finite number of iterations.
\par{REK2 costs the computational complexity comparable to $k$-means. Specifically, REK2 takes $\mathcal{O}(nd)$ and $\mathcal{O}(ndk)$ for computing the centroid matrix $\mathbf{M}$ and the cluster assignment matrix $\mathbf{P}$ in each iteration, respectively. Therefore, REK2 algorithm will converge in $\mathcal{O}(Tndk)$, where $T$ is the number of iterations until convergence. As Figure \ref{ConvergenceREK2_Scene13} and \ref{ConvergenceREK2_Caltech256} show, $T$ is very small. In summary, REK2 enjoys the low computational complexity as linear method while it is a nonlinear approach.}

 \begin{figure}
 \centering
 \includegraphics[width=7.5cm]{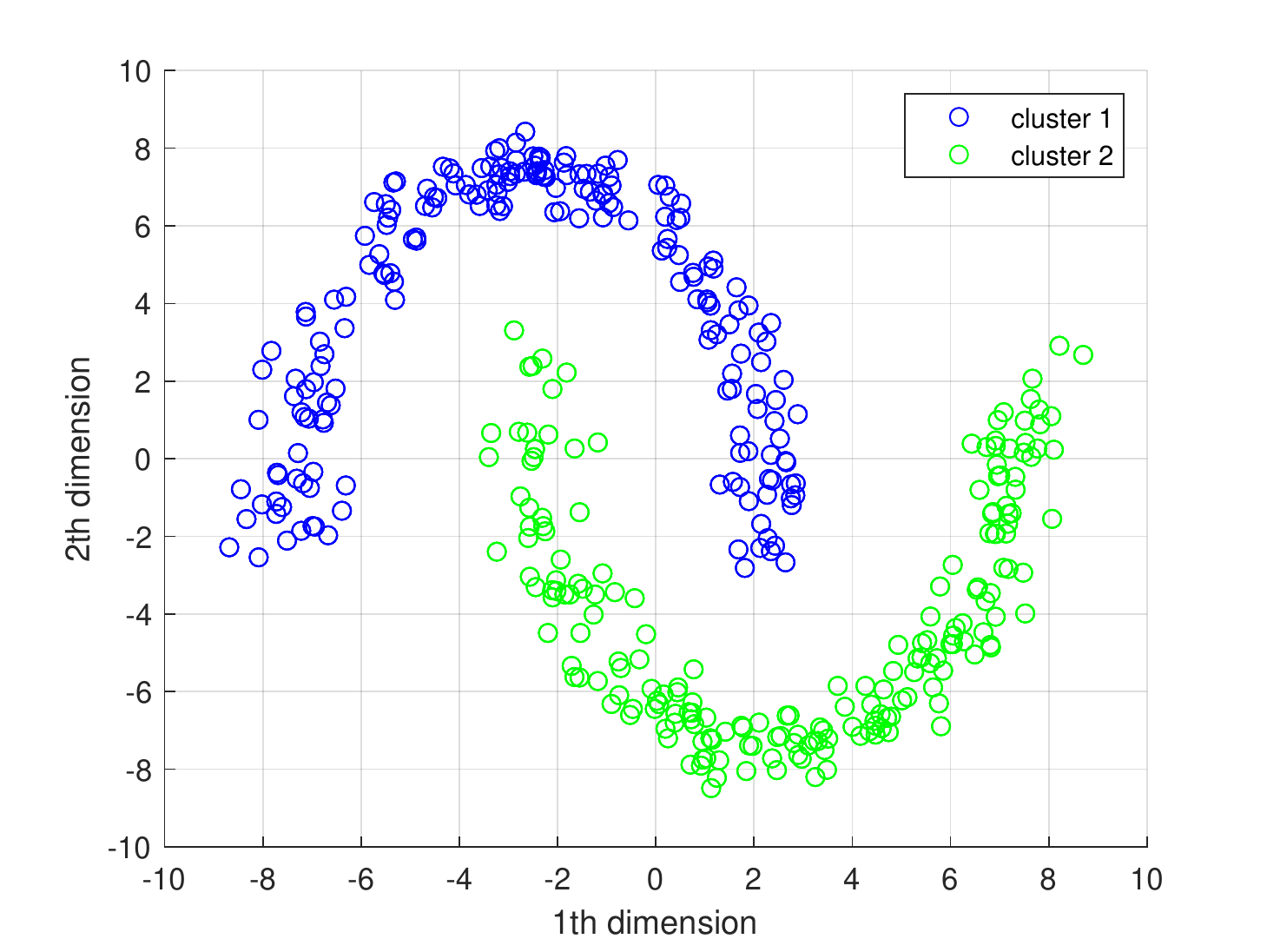}
 \caption{Half-moon data.}
 \label{HalfmoonData}
 \end{figure}

\begin{figure}
\centering
\subfigure[]{%
\includegraphics[width=5.55cm]{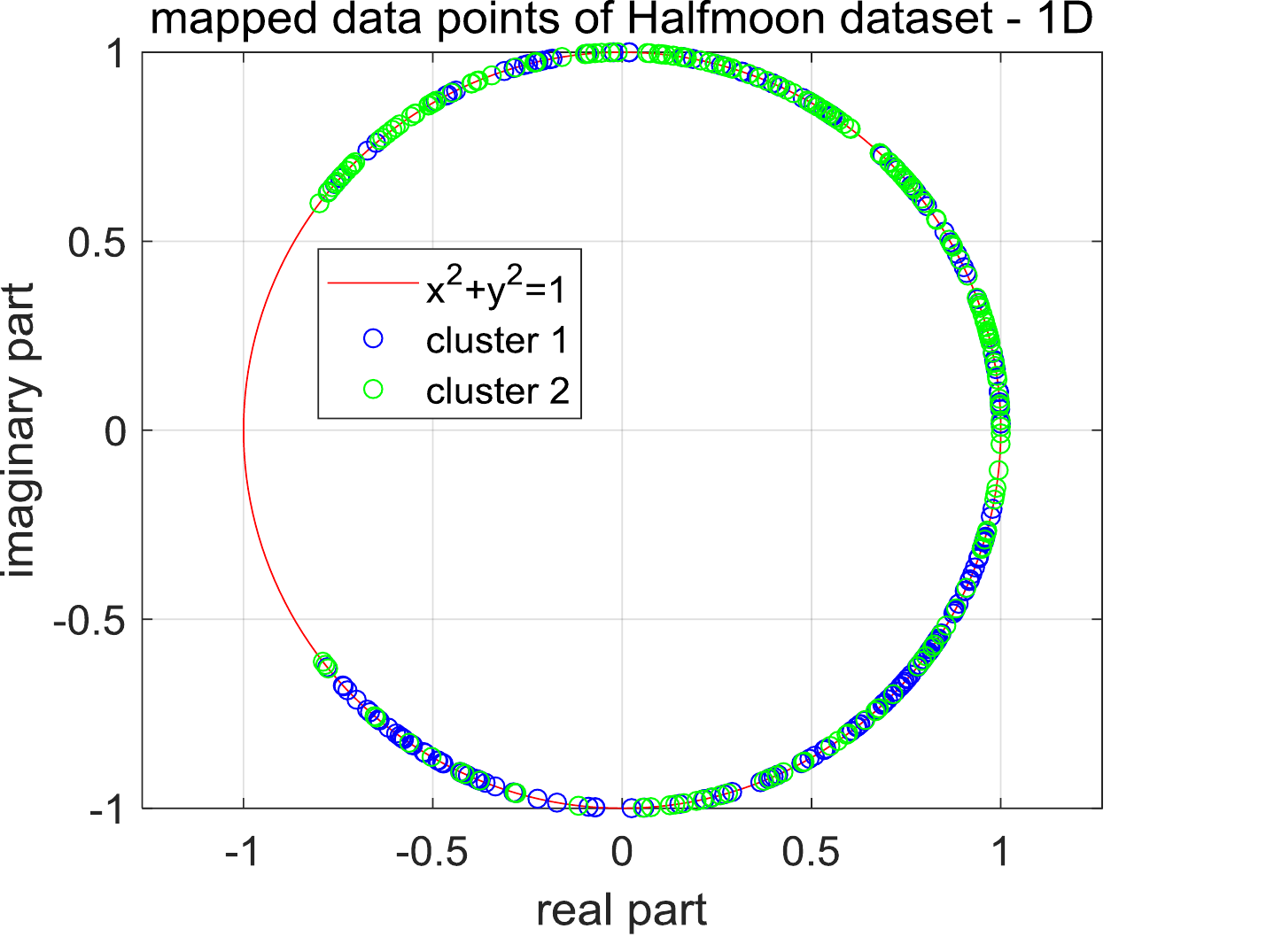}
\label{MapHalfmoon1D}}
\subfigure[]{%
\includegraphics[width=5.55cm]{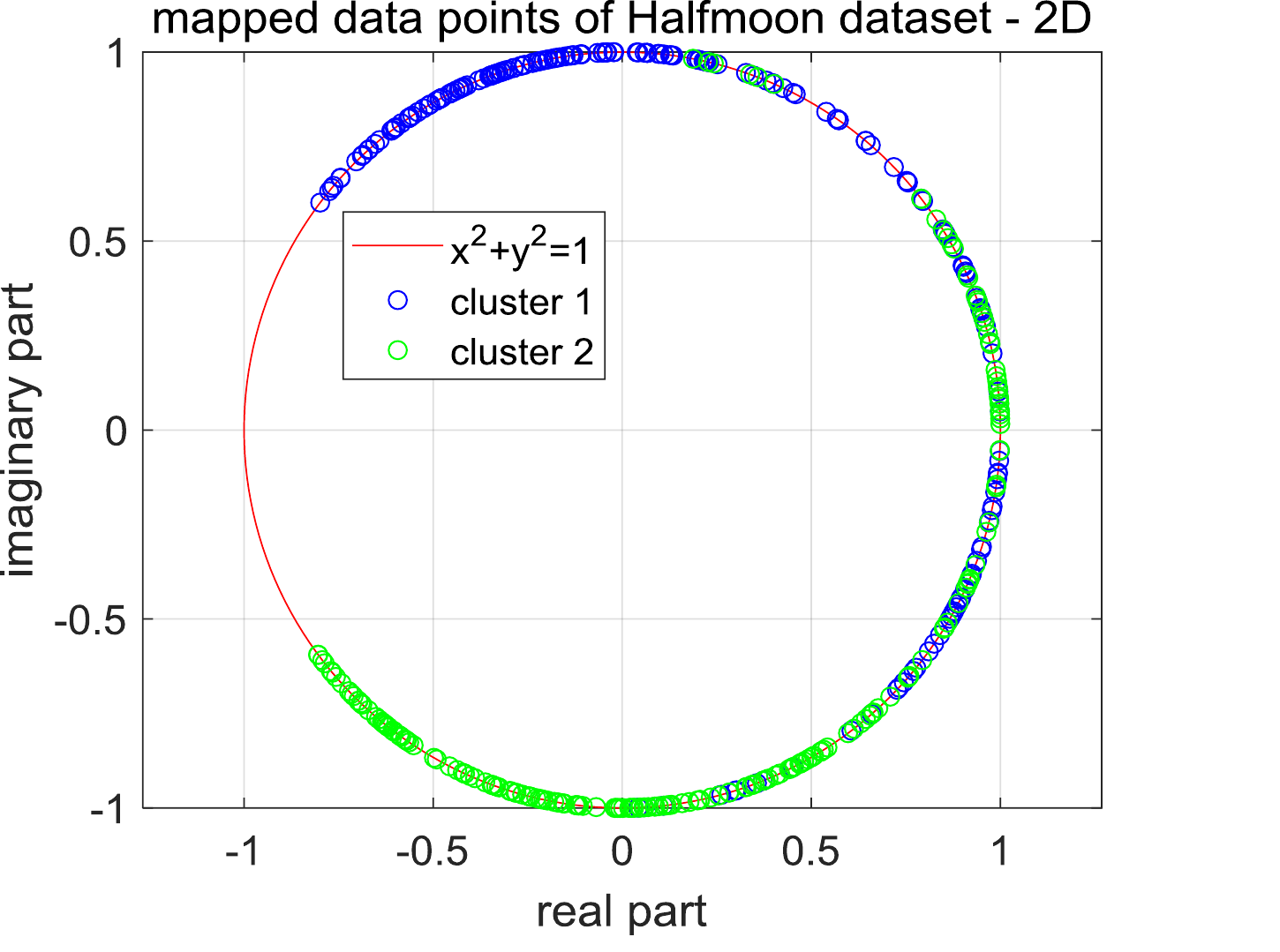}
\label{MapHalfmoon1D}}
\caption{Mapped data points of Halfmoon dataset. Because it is difficult to visualize the data points in 2-dimensional complex space, we show these data points over separate dimensions.}
\label{Map_Halfmoon}
\end{figure}

\begin{figure*}
\centering
\subfigure[]{%
\includegraphics[width=5.55cm]{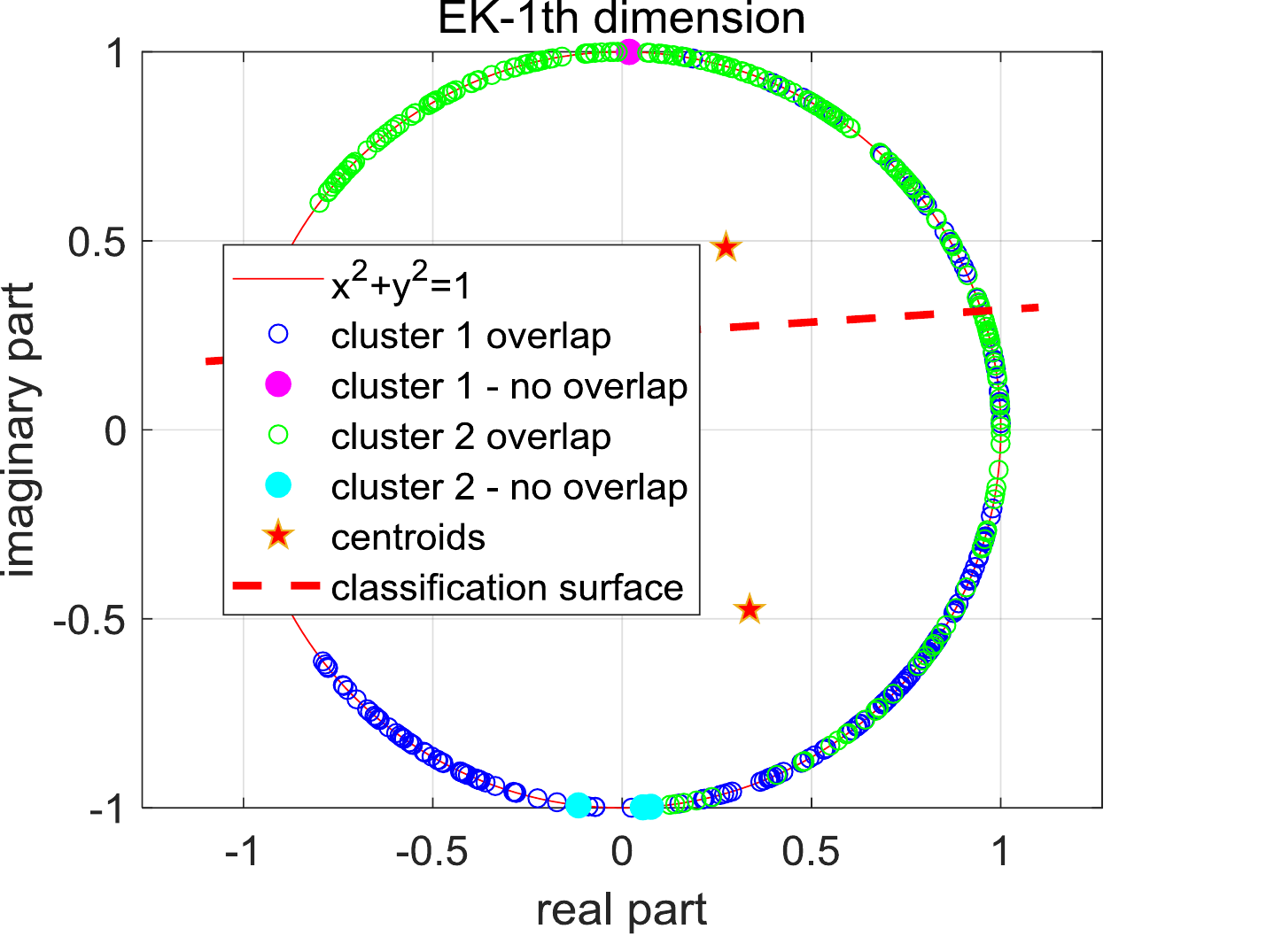}
\label{Halfmoon1DEK}}
\subfigure[]{%
\includegraphics[width=5.55cm]{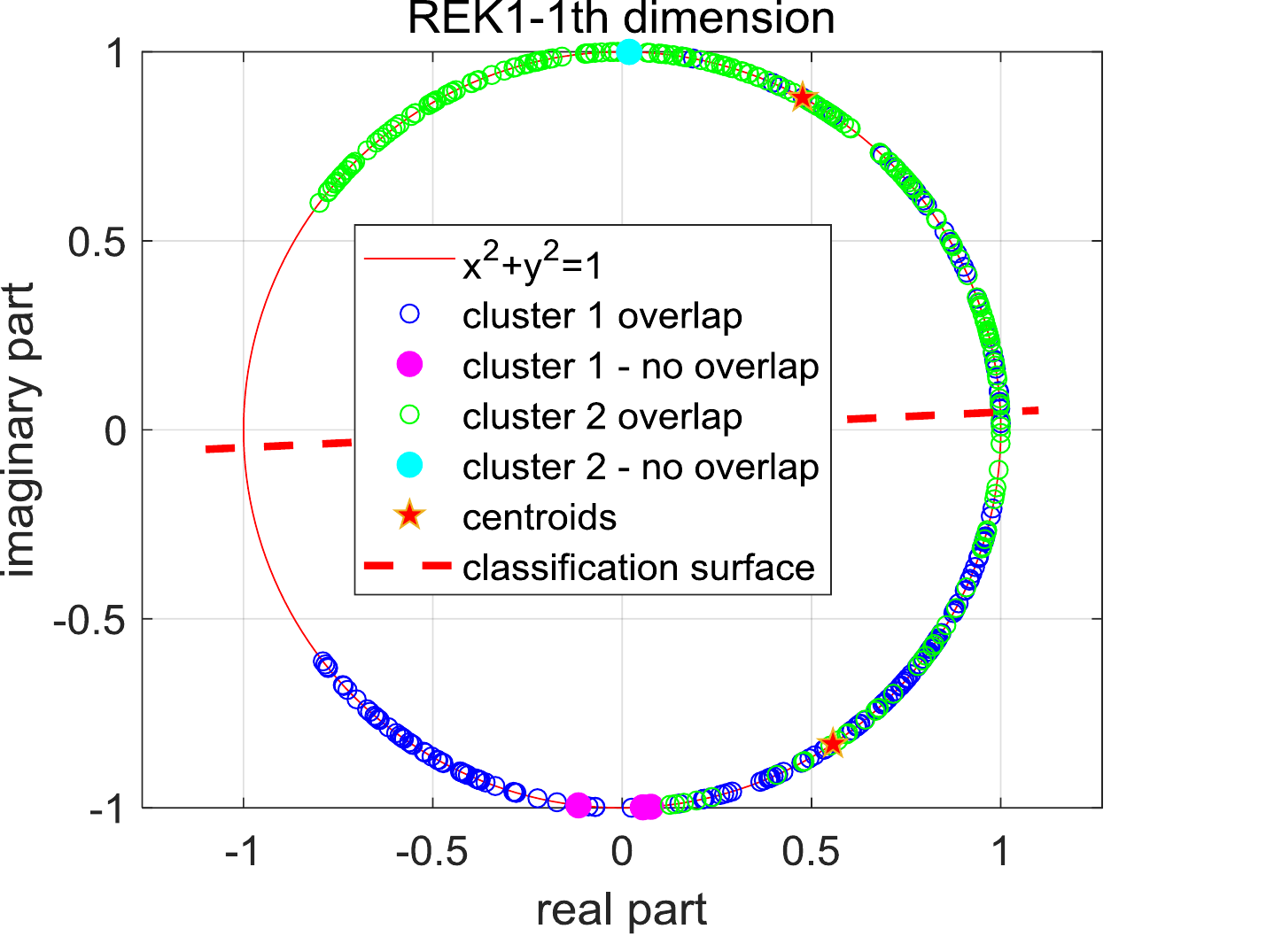}
\label{Halfmoon1DREK1}}
\subfigure[]{%
\includegraphics[width=5.55cm]{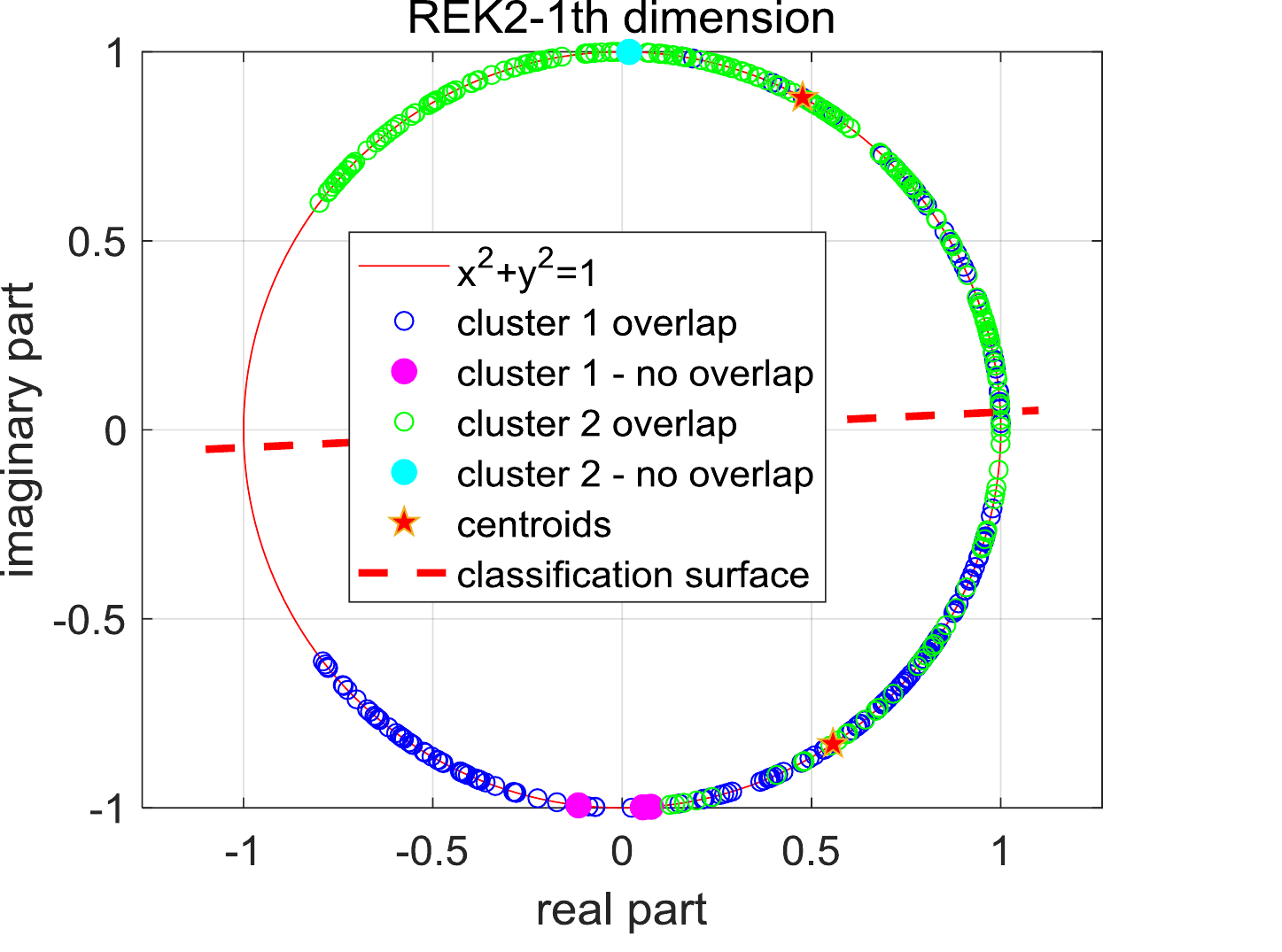}
\label{Halfmoon1DREK2}}
\quad
\subfigure[]{%
\includegraphics[width=5.55cm]{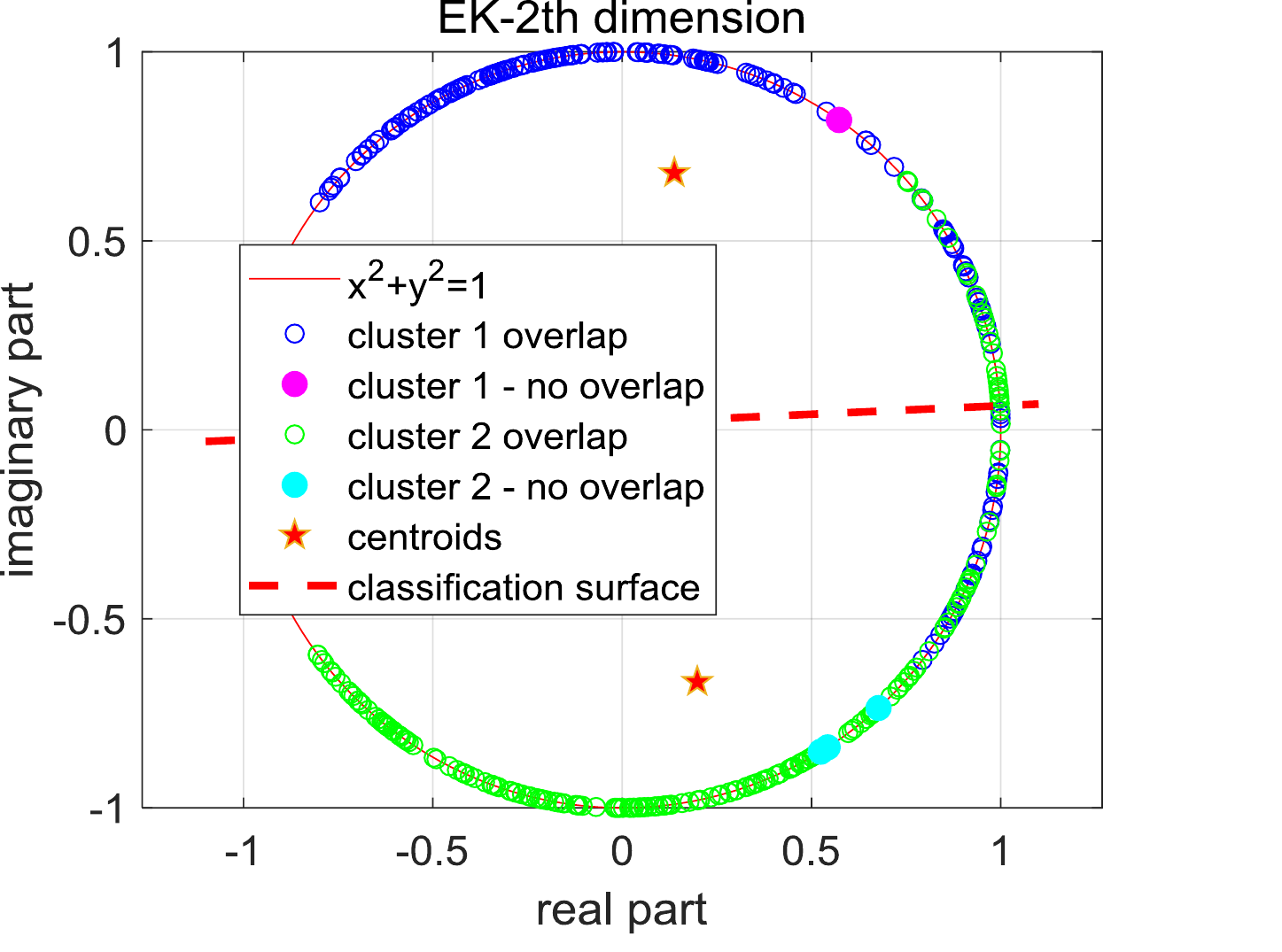}
\label{Halfmoon2DEK}}
\subfigure[]{%
\includegraphics[width=5.55cm]{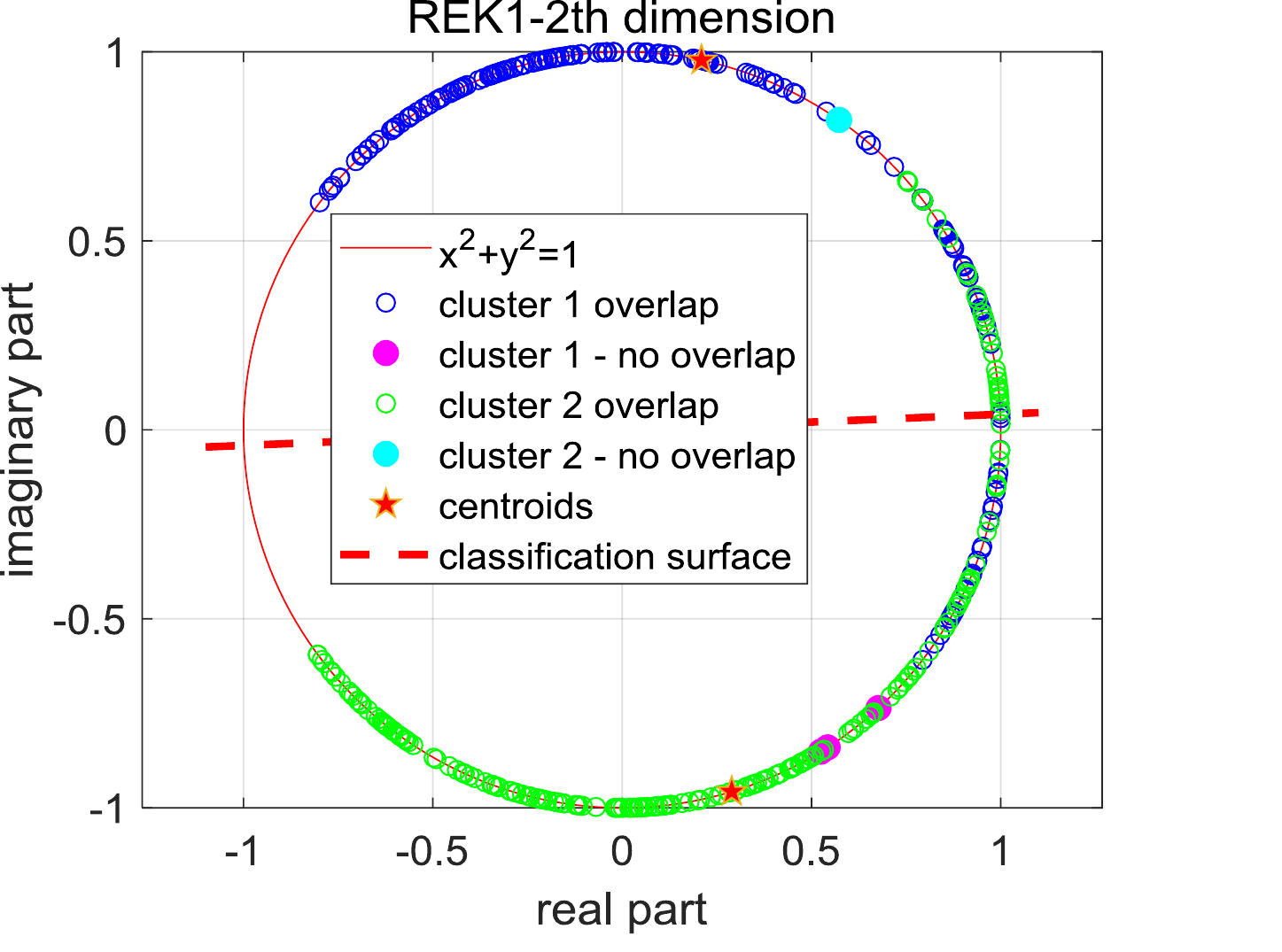}
\label{Halfmoon2DREK1}}
\subfigure[]{%
\includegraphics[width=5.55cm]{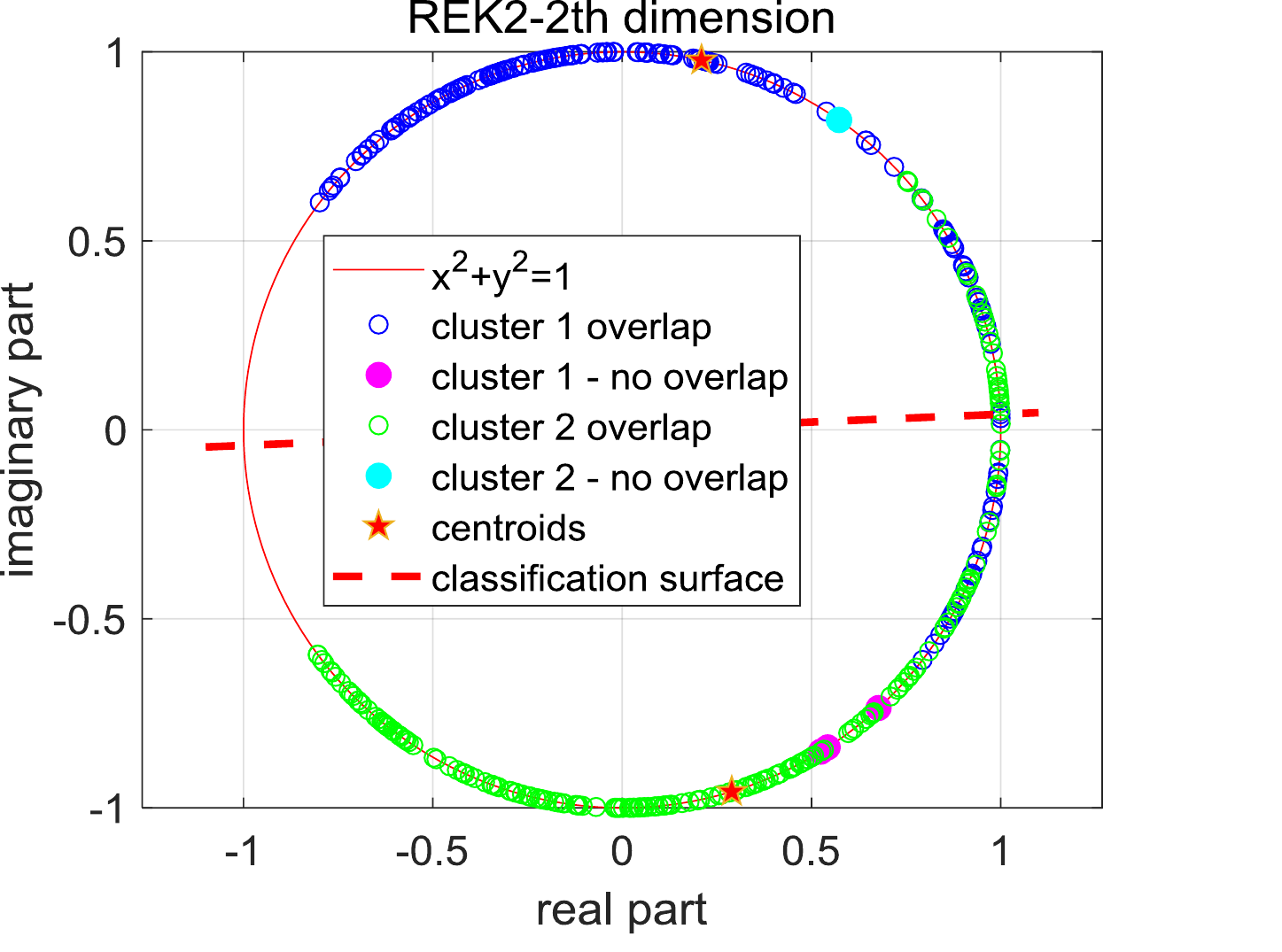}
\label{Halfmoon2DREK2}}
\caption{Mapped data points, centroids obtained by EulerK, REK1 and REK2 on Halfmoon dataset, respectively and corresponding clustering surfaces. Because it is difficult to visualize the data points in 2-dimensional complex space, we show the above items over separate dimensions.}
\label{SimulationResult}
\end{figure*}

\begin{figure*}
\centering
\subfigure[]{%
\includegraphics[width=5.55cm]{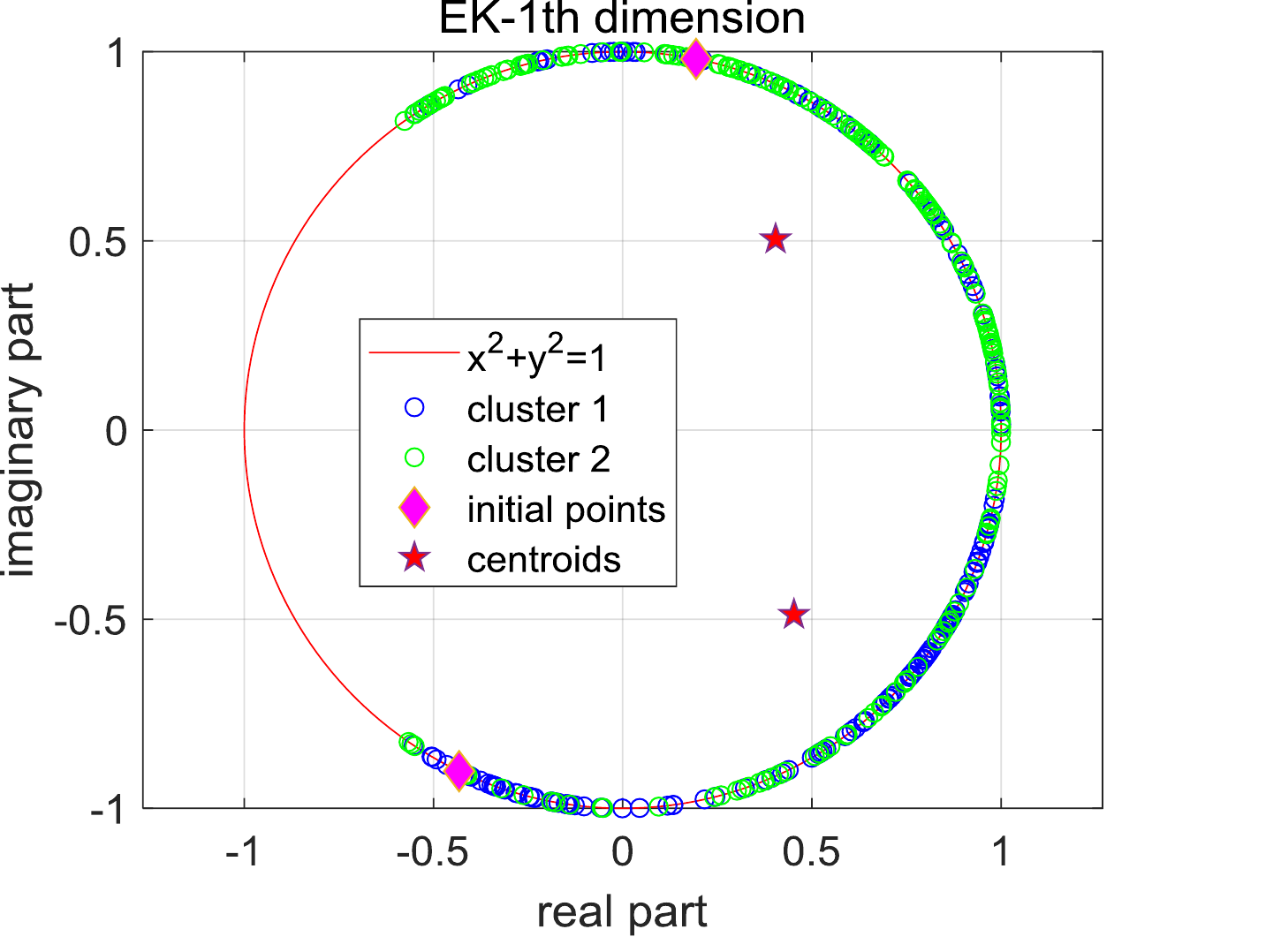}
\label{Halfmoon1DEK0}}
\subfigure[]{%
\includegraphics[width=5.55cm]{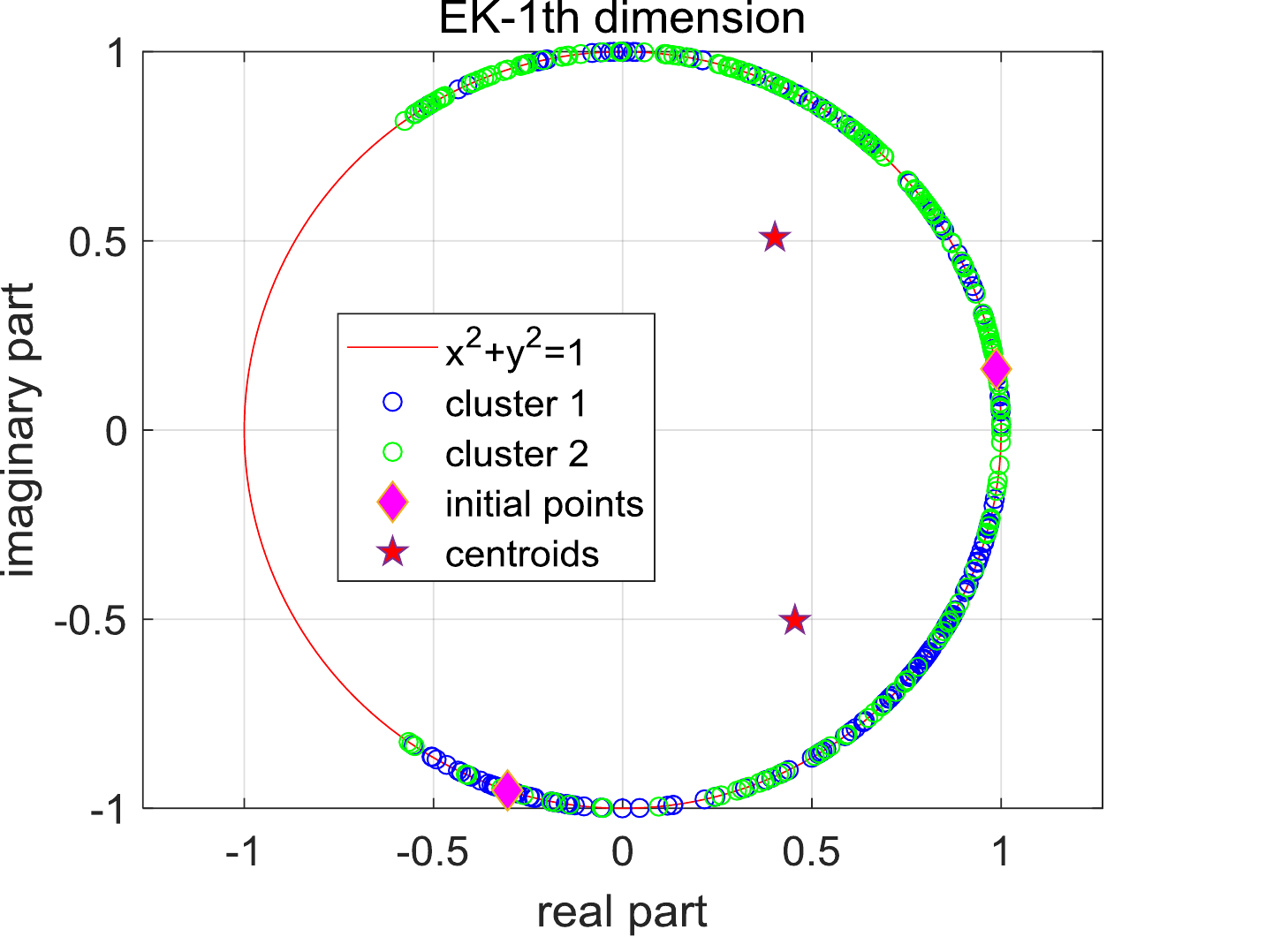}
\label{Halfmoon1DREK10}}
\subfigure[]{%
\includegraphics[width=5.55cm]{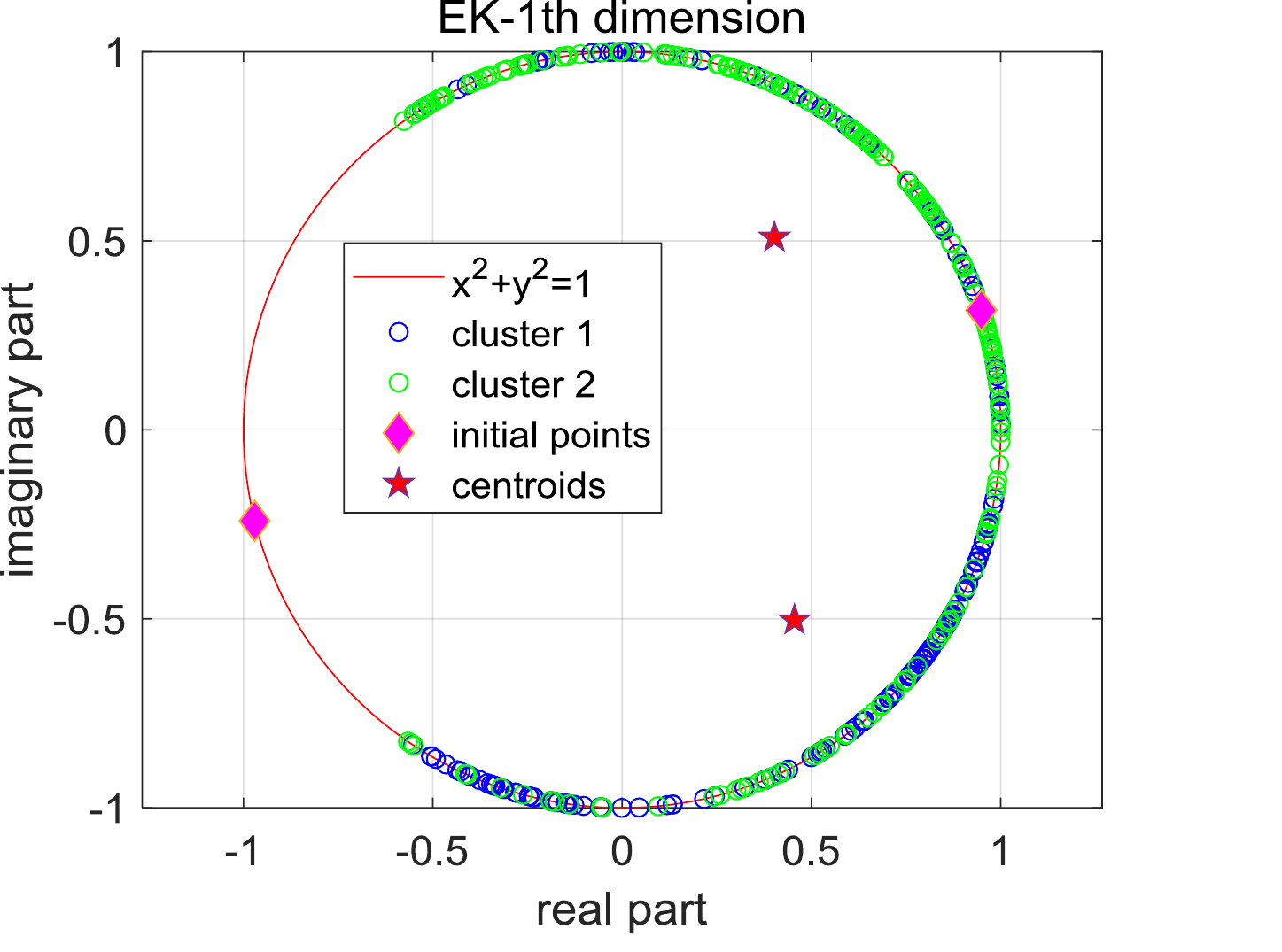}
\label{Halfmoon1DREK20}}
\quad
\subfigure[]{%
\includegraphics[width=5.55cm]{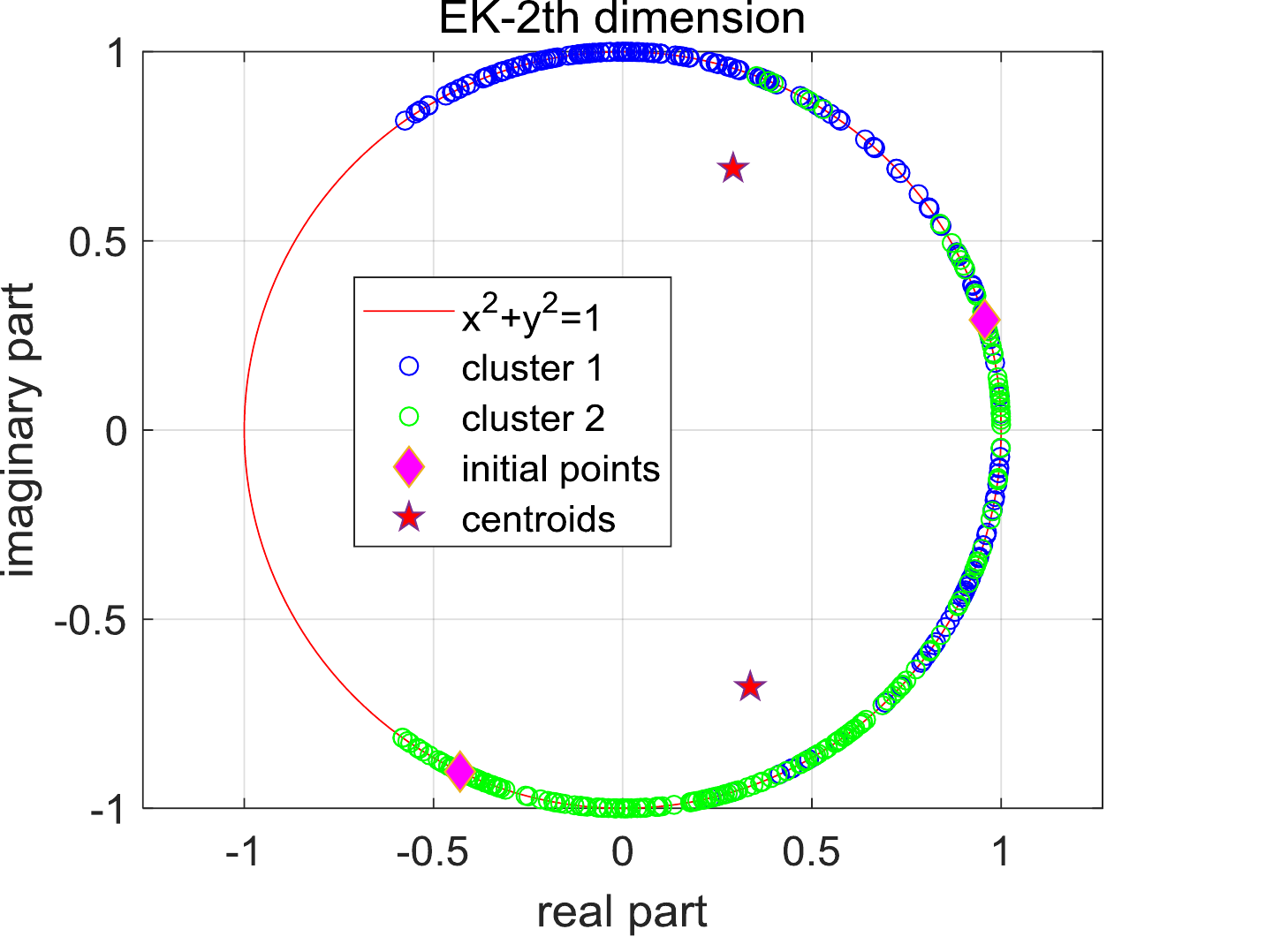}
\label{Halfmoon2DEK0}}
\subfigure[]{%
\includegraphics[width=5.55cm]{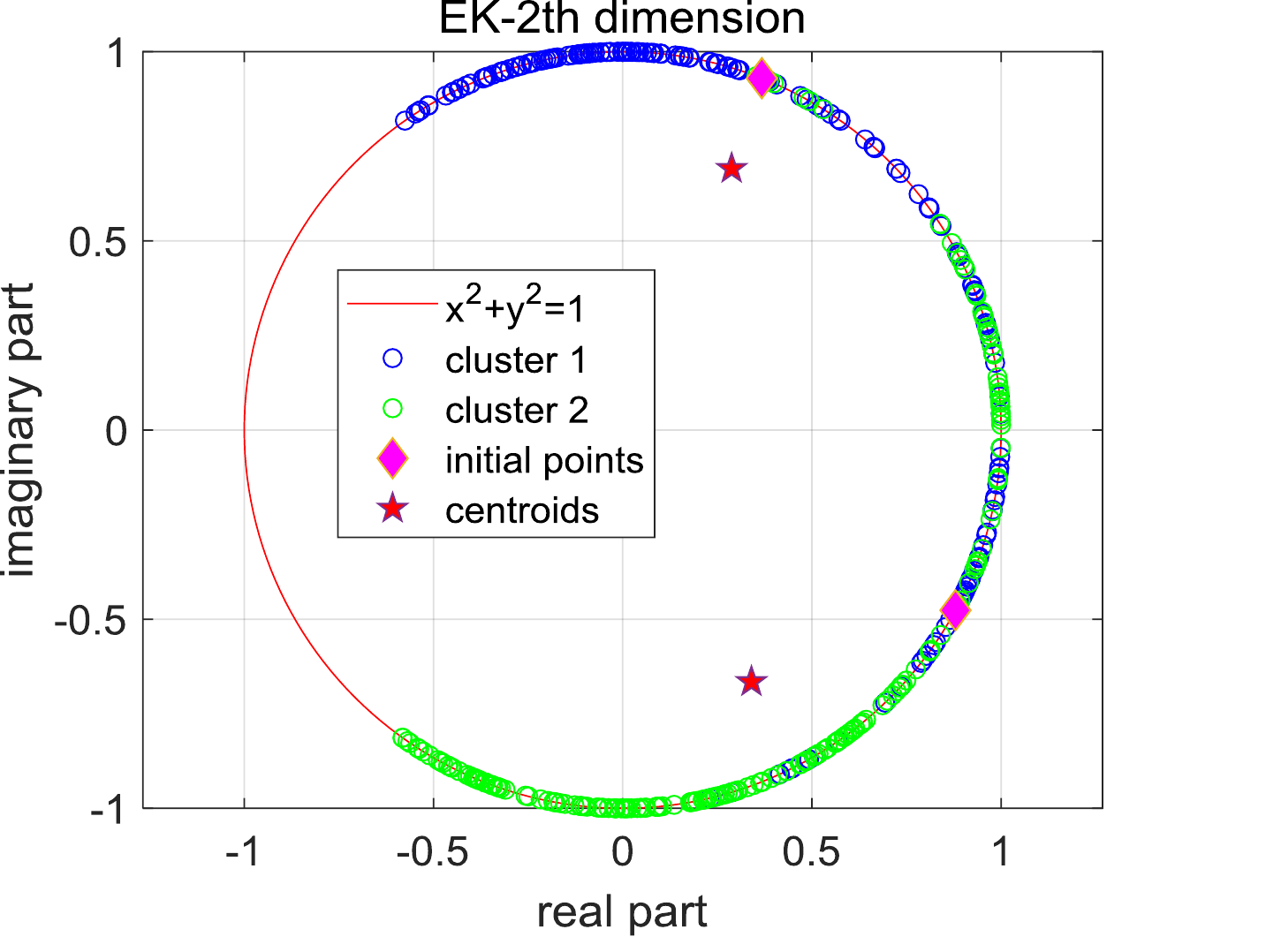}
\label{Halfmoon2DREK10}}
\subfigure[]{%
\includegraphics[width=5.55cm]{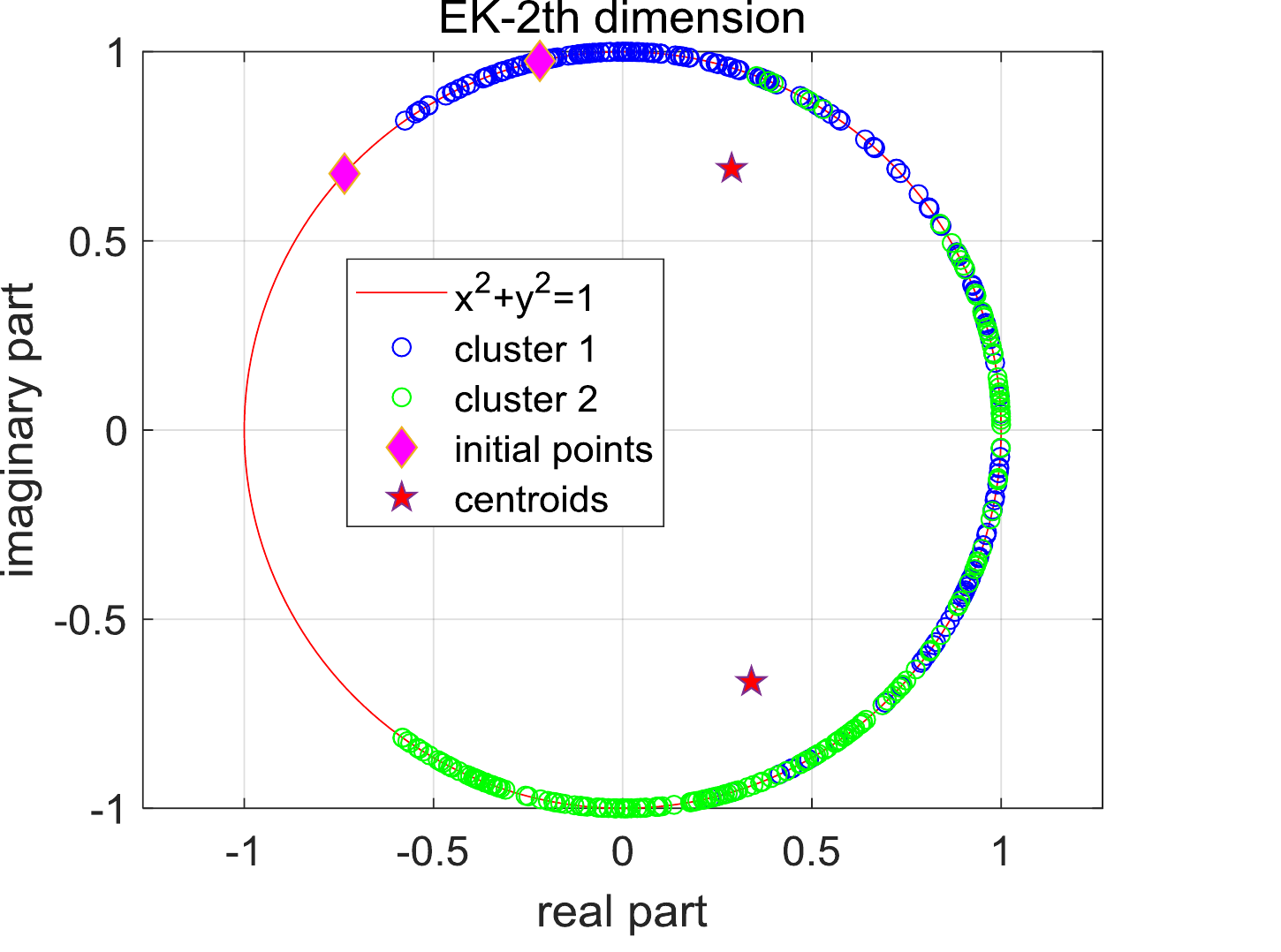}
\label{Halfmoon2DREK20}}
\caption{Centroids obtained by EulerK over 1th initialization. Because it is difficult to visualize the data points in 2-dimensional complex space, we show the centroids over separate dimensions.}
\label{SimulationResultIni1}
\end{figure*}

\begin{figure*}
\centering
\subfigure[]{%
\includegraphics[width=5.55cm]{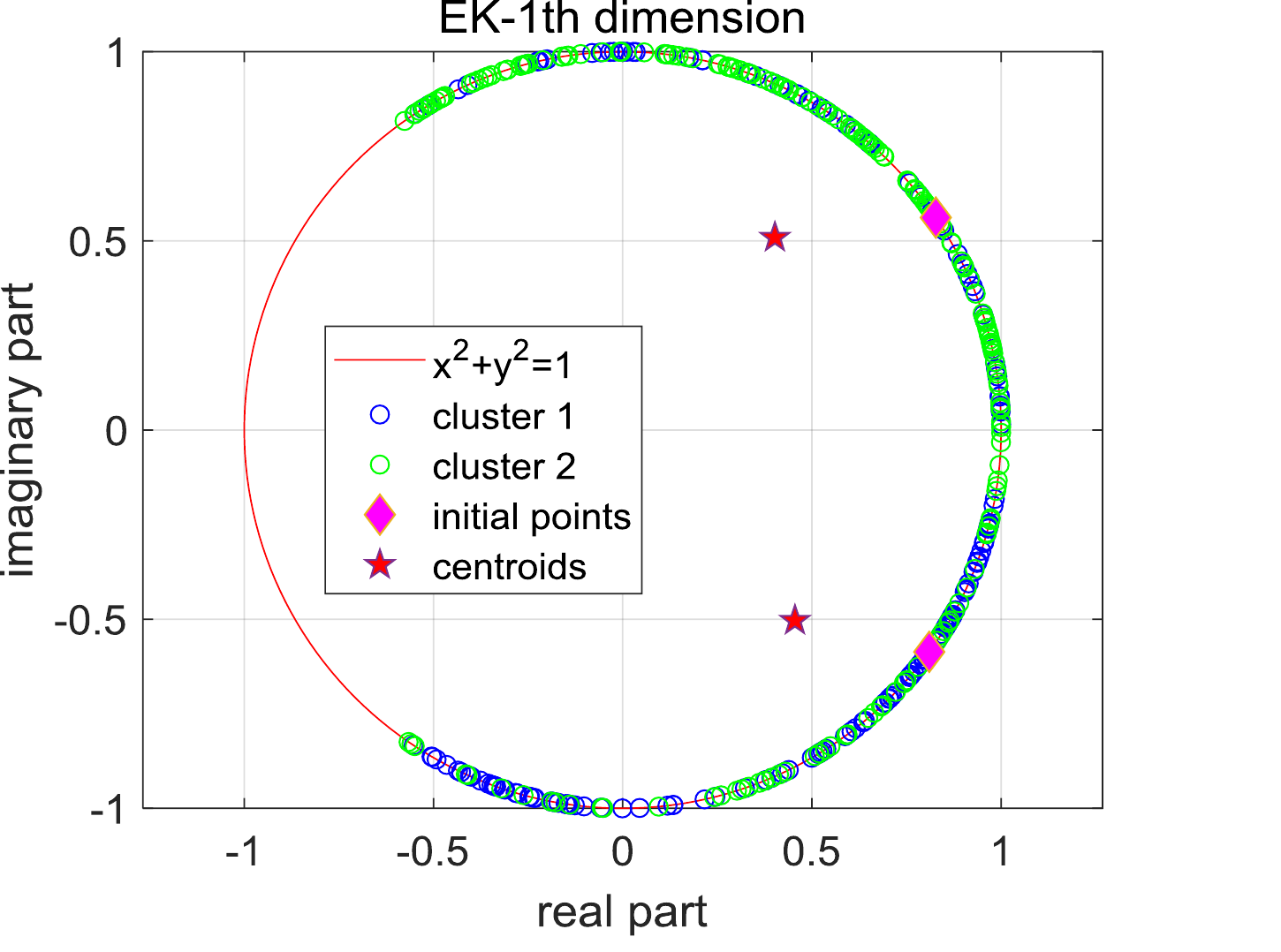}
\label{Halfmoon1DEK00}}
\subfigure[]{%
\includegraphics[width=5.55cm]{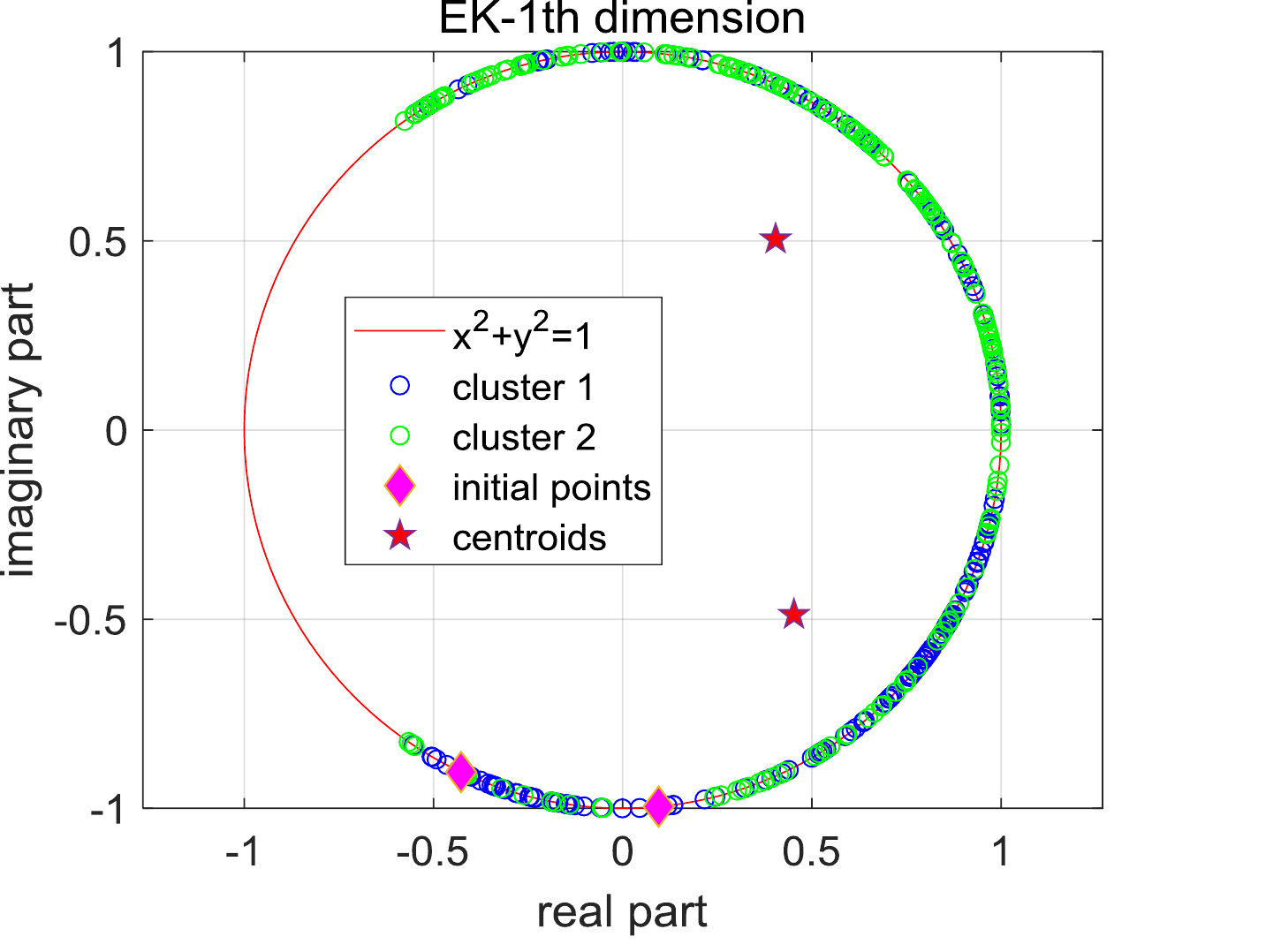}
\label{Halfmoon1DREK100}}
\subfigure[]{%
\includegraphics[width=5.55cm]{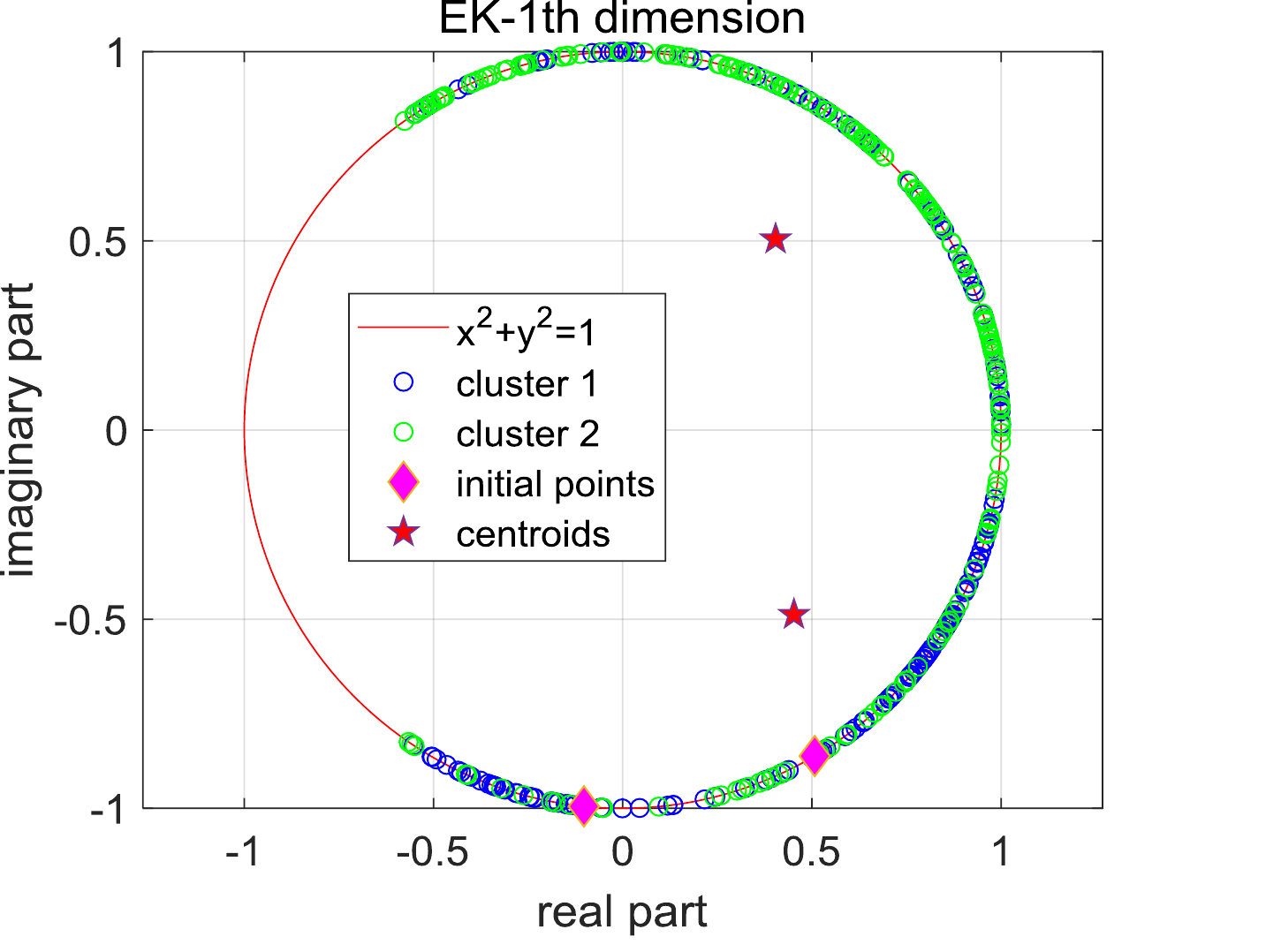}
\label{Halfmoon1DREK200}}
\quad
\subfigure[]{%
\includegraphics[width=5.55cm]{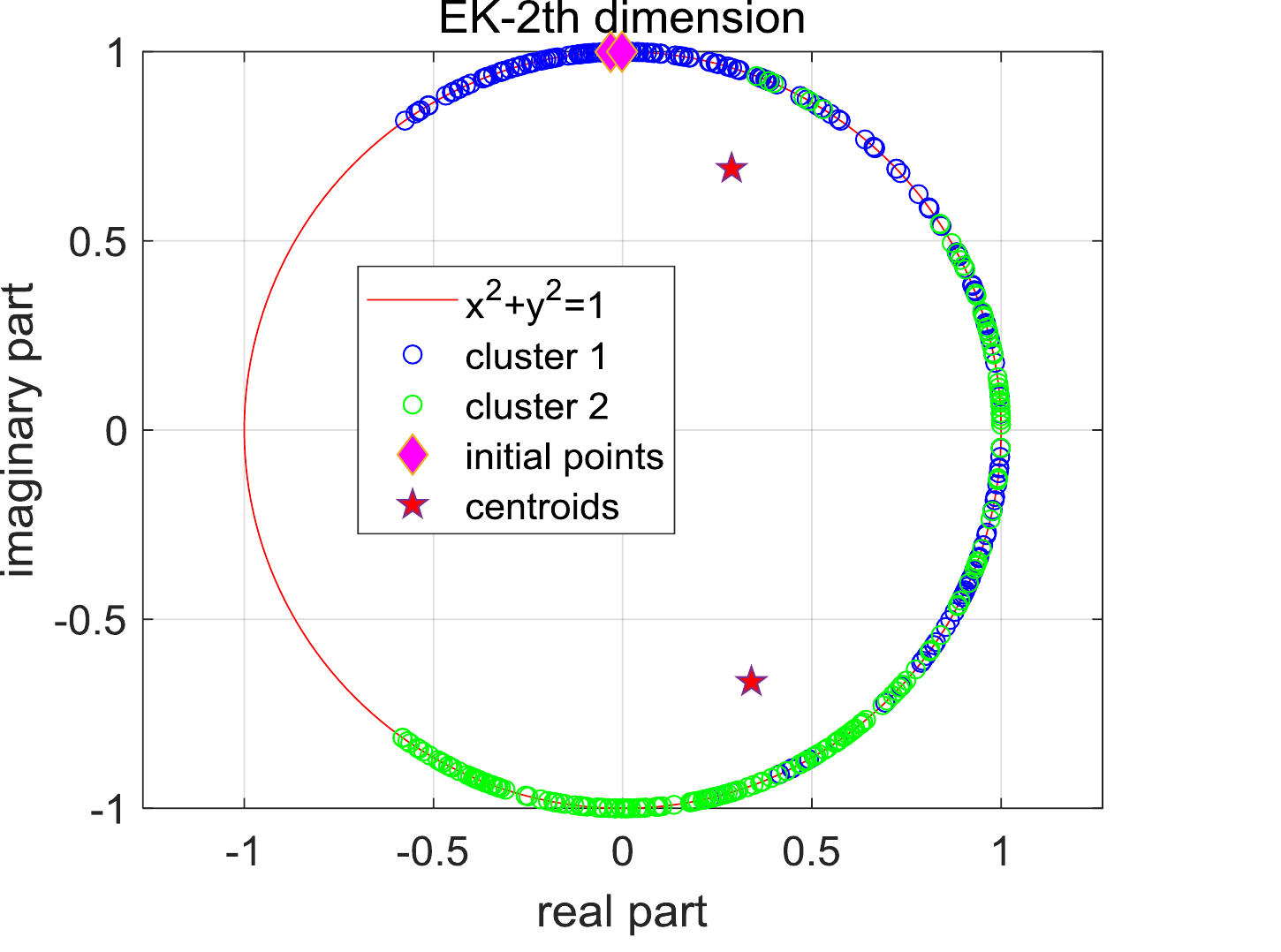}
\label{Halfmoon2DEK00}}
\subfigure[]{%
\includegraphics[width=5.55cm]{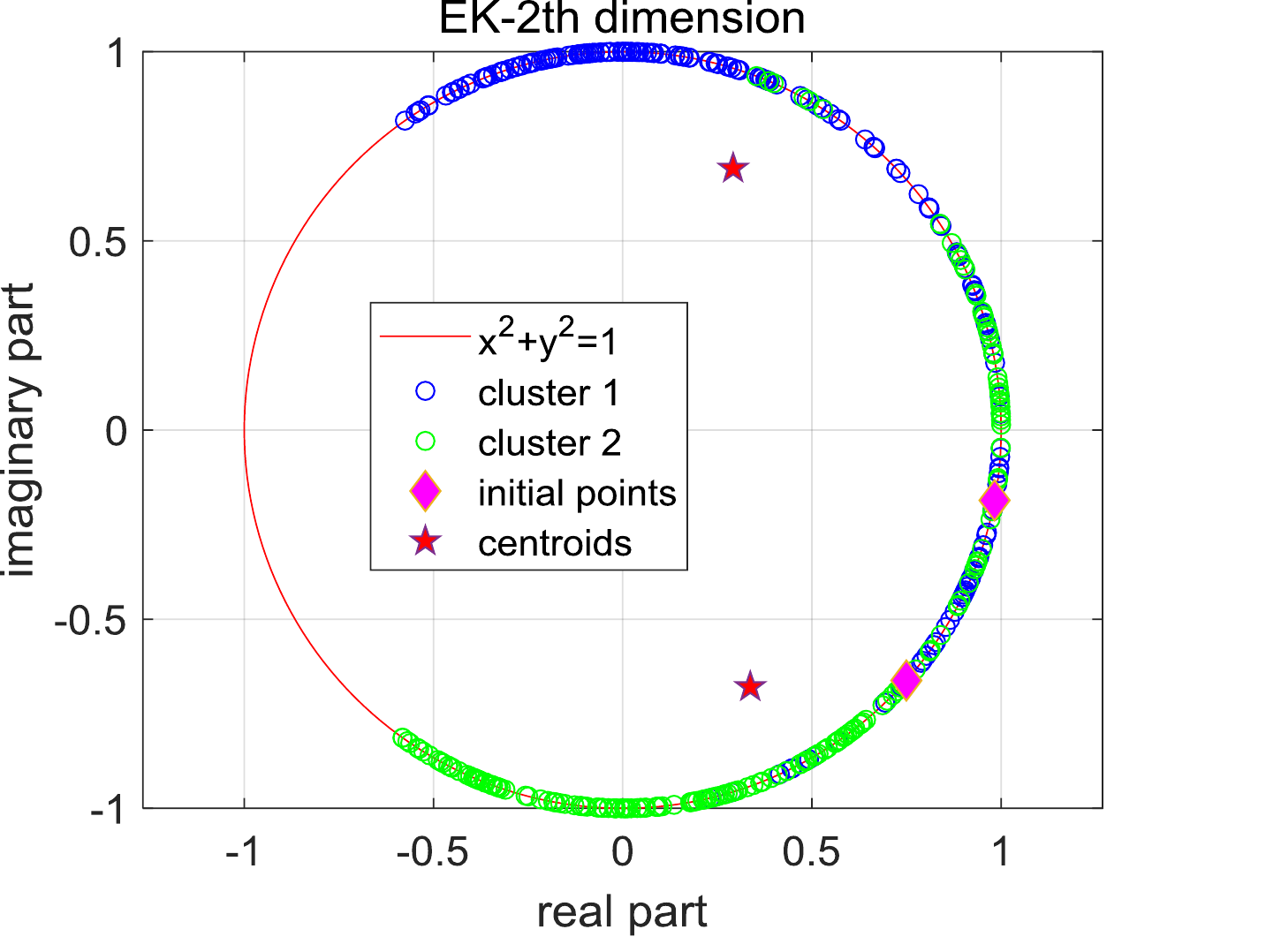}
\label{Halfmoon2DREK100}}
\subfigure[]{%
\includegraphics[width=5.55cm]{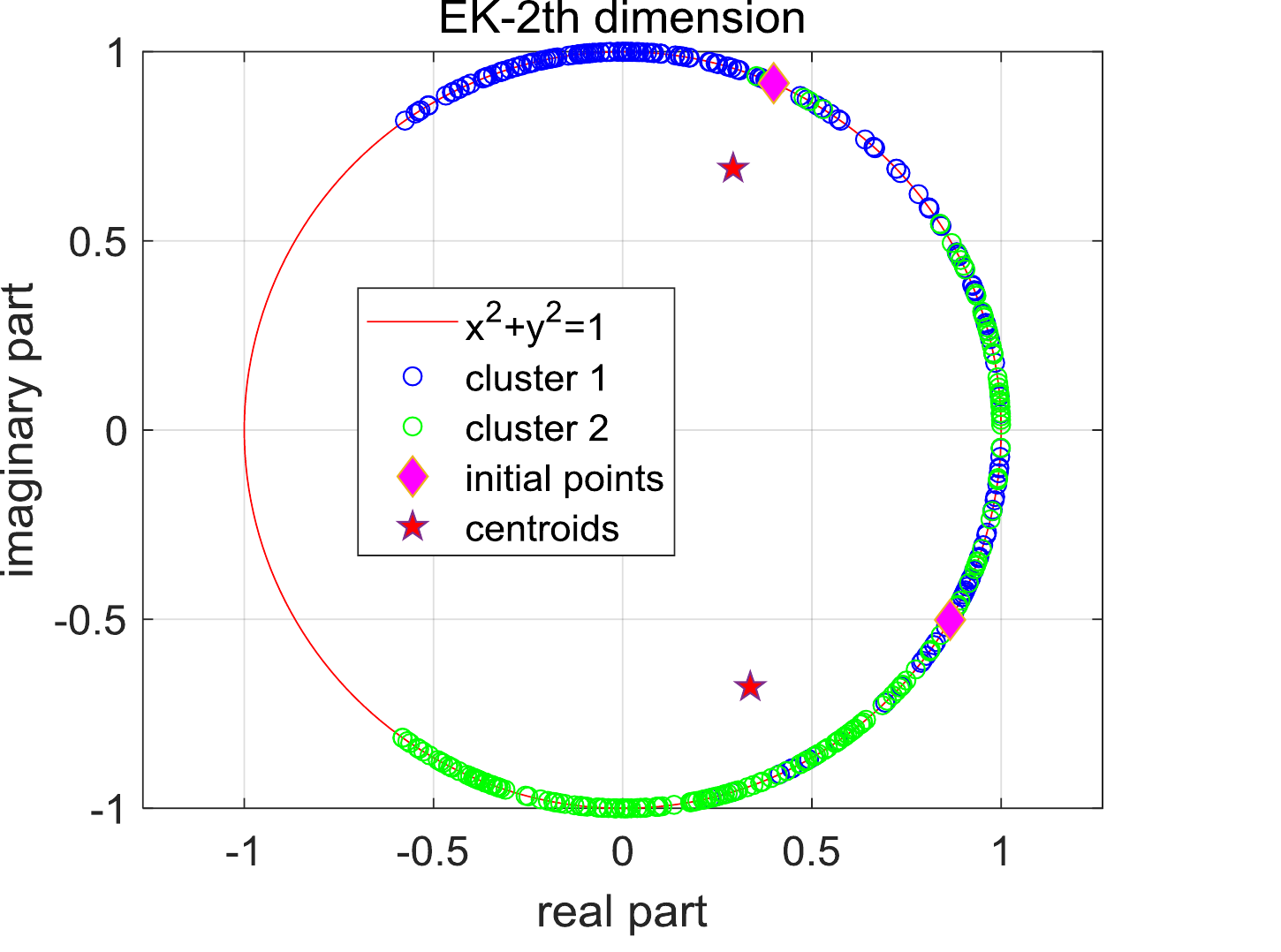}
\label{Halfmoon2DREK200}}
\caption{Centroids obtained by EulerK over 2th initialization. Because it is difficult to visualize the data points in 2-dimensional complex space, we show the centroids over separate dimensions.}
\label{SimulationResultIni2}
\end{figure*}

\section{Experiments}
In this section, we validate the effectiveness of our proposed REK1 and REK2 over one synthetic dataset and eleven real datasets, respectively.
\subsection{Simulation Study}
In this part, we give a synthetic example to intuitively visualize the interesting phenomenon occuring in EulerK and illustrate the effectiveness of our proposed REK1 and REK2. To achieve such goal, we first randomly generate a 2-dimensional synthetic data set, i.e., Halfmoon data which consists of 2 clusters and show it in Figure \ref{HalfmoonData}.
\subsubsection{Deviation Degree and Centroids}
We perform EulerK, REK1 and REK2 on such dataset, respectively and next show the mapped data in Figure \ref{Map_Halfmoon}, centroids obtained by these three methods and the final partition in Figure \ref{SimulationResult}, respectively. Because it is difficult to visualize the data points in 2-dimensional complex space, we show the above items over separate dimensions. From Figure \ref{Halfmoon1DEK} and \ref{Halfmoon2DEK}, we can see that the centroids obtained by EulerK indeed deviate from the hyper-sphere surface with a large margin, i.e., not residing in the support domain of mapped data points and in some sense, actually are outliers. To quantitatively measure such deviated degree, we define a criterion named \emph{Deviation Degree} $\kappa$ based on the distance between each centroid $\mathbf{m}=\mathbf{a}+i\mathbf{b}$ and the unit hyper-sphere surface and formulate it as
\begin{equation}
\kappa = 1-\frac{\sum_{i=1}^{d}\mathbf{a}_{i}^2+\mathbf{b}_{i}^2}{\sqrt{d}}.
\end{equation}
Obviously, the larger the value of $\kappa$ is, the higher the deviation of centroids to unit hyper-sphere surface is. For the centroids obtained by EulerK on this dataset, $\kappa = 0.298$ quantitatively reflects their large degree of deviation to the distribution of mapped data.
Next, Figure \ref{Halfmoon1DREK1}\ref{Halfmoon2DREK1} and \ref{Halfmoon1DREK2}\ref{Halfmoon2DREK2} indicate that prototypes respectively acquired by REK1 and REK2 strictly reside on the mapped space and almost reveal the distribution of the mapped data points. Moreover, given the centroids, we calculate the classification surfaces of EulerK, REK1, REK2 respectively according to \textbf{Lemma} \ref{lemma1}. and \ref{lemma2}., respectively and show them in Figure \ref{SimulationResult} dimension by dimension. From Figure \ref{SimulationResult}, we can see that the classification surfaces corresponding to REK1 and REK2 perform better than that corresponding to EulerK, so we can conclude that the classification surfaces of REK1 and REK2 have better generalization ability than that corresponding to EulerK. This further reinforces the rationality of our questioning of this phenomenon occuring in EulerK and validate the effectiveness of our proposed REK1 and REK2.

\begin{lemma} \label{lemma1}
Given two centroids $\mathbf{m}_p=\mathbf{a}_{p}+i\mathbf{b}_p$ and $\mathbf{m}_{q}=\mathbf{a}_{q}+i\mathbf{b}_q$ in $d$-dimensional complex space, the classification surface of $p$-th cluster and $q$-th cluster corresponding to EulerK can be formulated as
\begin{equation}
\begin{aligned}
&\sum_{l=1}^{d}(\mathbf{a}_{p,l}-\mathbf{a}_{q,l})\mathbf{a}_{l}-(\mathbf{a}_{p,l}^2-\mathbf{a}_{q,l}^2)\\
&+\sum_{l=1}^{d}(\mathbf{b}_{p,l}-\mathbf{b}_{q,l})\mathbf{b}_{l}-(\mathbf{b}_{p,l}^2-\mathbf{b}_{q,l}^2)=0.
\end{aligned}
\end{equation}
\end{lemma}
\begin{proof}
Given a data point $x=\mathbf{a}+i\mathbf{b}$ residing on the unit hyper-sphere surface of $d$-dimensional complex space, let $||x-\mathbf{m}_p||^2=||x-\mathbf{m}_q||^2$, we have $\sum_{l=1}^{d}(\mathbf{a}_l-\mathbf{a}_{p,l})^2+\sum_{l=1}^{d}(\mathbf{b}_l-\mathbf{b}_{p,l})^2=\sum_{l=1}^{d}(\mathbf{a}_l-\mathbf{a}_{q,l})^2+\sum_{l=1}^{d}(\mathbf{b}_l-\mathbf{b}_{q,l})^2$
, unfolding this equation, we then get the classification surface formula
$\sum_{l=1}^{d}(\mathbf{a}_{p,l}-\mathbf{a}_{q,l})\mathbf{a}_{l}-(\mathbf{a}_{p,l}^2-\mathbf{a}_{q,l}^2)
+\sum_{l=1}^{d}(\mathbf{b}_{p,l}-\mathbf{b}_{q,l})\mathbf{b}_{l}-(\mathbf{b}_{p,l}^2-\mathbf{b}_{q,l}^2)=0$.
\end{proof}

\begin{lemma} \label{lemma2}
Given two cluster centroids $\mathbf{m}_p=\mathbf{a}_{p}+i\mathbf{b}_p$ and $\mathbf{m}_{q}=\mathbf{a}_{q}+i\mathbf{b}_q$ in $d$-dimensional complex space, the classification surfaces of $p$th cluster and $q$th cluster corresponding to REK1 and REK2 can be formulated as
\begin{equation}
\sum_{l=1}^{d}(\mathbf{a}_{p,l}-\mathbf{a}_{q,l})\mathbf{a}_{l}+
\sum_{l=1}^{d}(\mathbf{b}_{p,l}-\mathbf{b}_{q,l})\mathbf{b}_{l}=0.
\end{equation}
\end{lemma}
\begin{proof}
Given a data point $x=\mathbf{a}+i\mathbf{b}$ residing on the unit hyper-sphere surface of $d$-dimensional complex space, let $||x-\mathbf{m}_p||^2=||x-\mathbf{m}_q||^2$, we have $\sum_{l=1}^{d}(\mathbf{a}_l-\mathbf{a}_{p,l})^2+\sum_{l=1}^{d}(\mathbf{b}_l-\mathbf{b}_{p,l})^2=\sum_{l=1}^{d}(\mathbf{a}_l-\mathbf{a}_{q,l})^2+\sum_{l=1}^{d}(\mathbf{b}_l-\mathbf{b}_{q,l})^2$
. Unfolding this equation and utilizing the constraints of $\mathbf{a}_l^2+\mathbf{b}_l^2=1$ , we then get the classification surface formula
$\sum_{l=1}^{d}(\mathbf{a}_{p,l}-\mathbf{a}_{q,l})\mathbf{a}_{l}+
\sum_{l=1}^{d}(\mathbf{b}_{p,l}-\mathbf{b}_{q,l})\mathbf{b}_{l}=0$.
\end{proof}
%$i^{th}$

\subsubsection{Initialization on Hypersphere Surface}
In this part, we intend to verify that whether EulerK still obtains the outlier-like centroids if we initialize the centroids on the support domain of mapped data. To this end, we set up two different initialization schemes including randomly choosing $K$ data points on the unit hyper-sphere surface as initial centroids and randomly choosing $K$ mapped data points as the initial centroids. We implement EulerK on Halfmoon dataset over the above two different initialization settings where each setting corresponds to three random choice and show the acquired centroids in Figure \ref{SimulationResultIni1} and \ref{SimulationResultIni2}, respectively. Figure \ref{SimulationResultIni1} and \ref{SimulationResultIni2} show that EulerK indeed still acquire the centroids that deviate from the support domain, thus further verifying our doubts about reasonableness of utilizing such centroids to partition the mapped data.

\subsubsection{Different $K$}
In this part, we show the relationship between $\kappa$ and centroids obtained by EulerK over different $K$ in Figure \ref{DeviKHalfmoonData}. Figure \ref{DeviKHalfmoonData} illustrates that $\kappa$ decreases as $K$ increases, this phenomenon shows that EulerK obtains outlier-like centroids especially over small number of centroids and also reflect a fact that it actually needs more centroids to represent the nonlinear dataset.
\begin{figure}
 \centering
 \includegraphics[width=7.5cm]{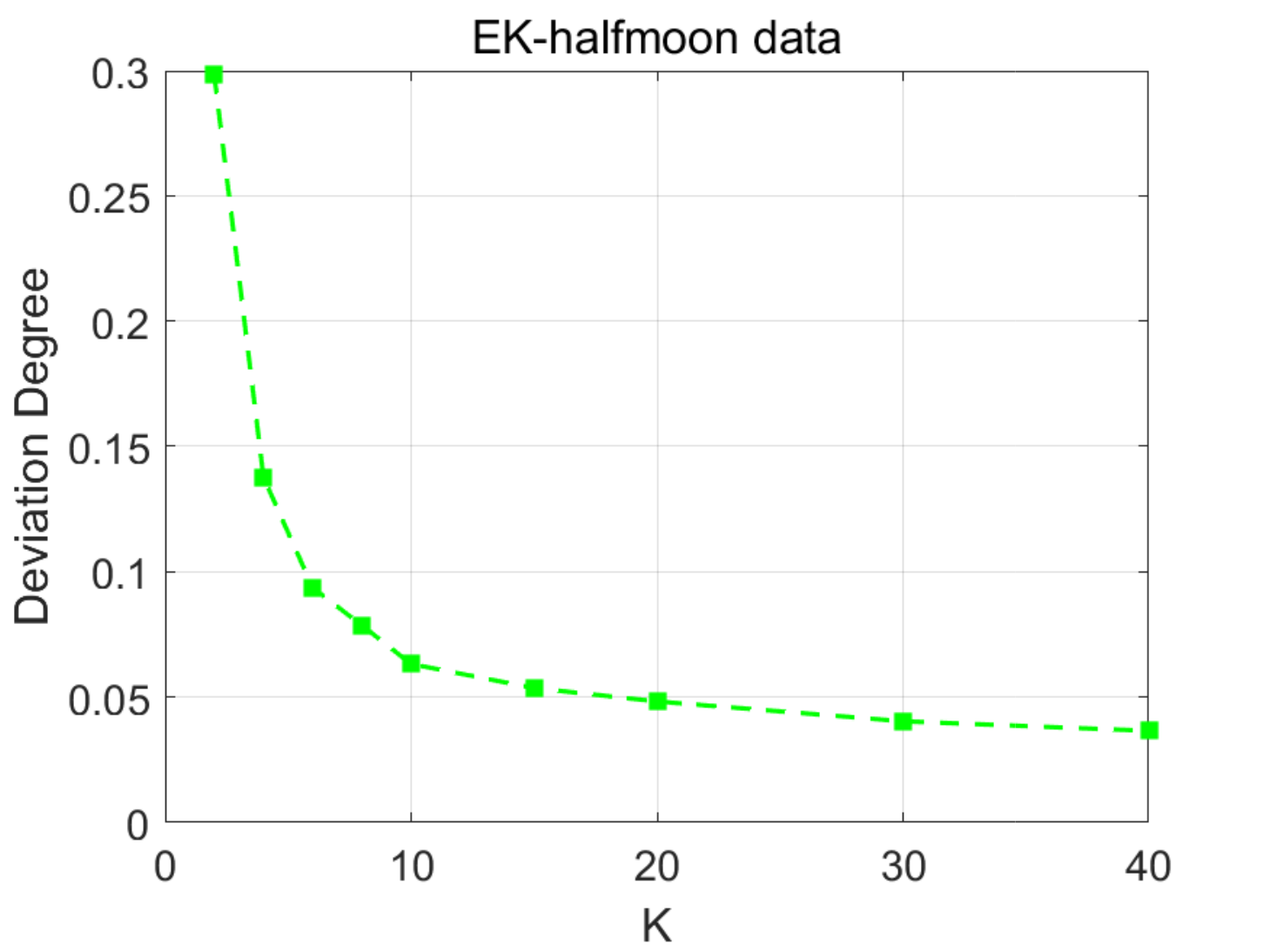}
 \caption{Half-moon data.}
 \label{DeviKHalfmoonData}
 \end{figure}

\subsection{Experimental Settings}
\subsubsection{Datasets}
Because we do not focus on the large-scale clustering problem, we choose six medium scale datasets used in \cite{EulerKmeans2}. Moreover, we choose several real datasets commonly used in kernel based methods. In summary, we evaluate our proposed methods REK1 and REK2 on eleven public datasets including low-dimensional ones of
wine\footnote[1]{\url{https://archive.ics.uci.edu/ml/datasets/Wine}},
seeds \footnote[2]{\url{http://archive.ics.uci.edu/ml/datasets/seeds}},
yeast \footnote[3]{\url{http://archive.ics.uci.edu/ml/datasets/yeast}},
isolet\footnote[4]{\url{http://www.cad.zju.edu.cn/home/dengcai/Data/MLData.html}},
pendigits\footnote[5]{\url{http://odds.cs.stonybrook.edu/pendigits-dataset/}},
satimage\footnote[6]{\url{https://archive.ics.uci.edu/ml/datasets/Statlog+(Landsat+Satellite)}} and high-dimensional ones of Event \cite{Event}, Scenes 13 \cite{Scene 13}, Caltech 101 \cite{Caltech 101}, SImagenet \cite{SImagenet}, Caltech 256 \cite{Caltech 256}. Table \ref{StatisticOfDatasets} summarizes the statistic information of these datasets and we detail them as follows:

\begin{itemize}
\item wine: There are 210 subjects describing wines, which are grown in the same region but derived from three different cultivars.
\item seeds: This dataset comprised kernels belonging to three different varieties of wheat: Kama, Rosa and Canadian, 70 elements each. Each data point covers seven geometric real-valued continuous parameters of wheat kernels: area, perimeter, compactness, length of kernel, width of kernel, asymmetry coefficient, length of kernel groove.
\item yeast: This dataset includes 1484 instances to predict the cellular localization sites of proteins.
\item isolet: This dataset contains 150 subjects who spoke the name of each letter of the alphabet twice.
\item pendigits: The original pendigits (Pen-Based Recognition of Handwritten Digits) dataset from UCI machine learning repository is a multiclass classification dataset having 16 integer attributes and 10 classes ('0'-'9'). The digit database is created by collecting 250 samples from 44 writers.
\item satimage: This database consists of the multi-spectral values of pixels in $3\times3$ neighborhood in a satellite image. The database is a tiny sub-area of a scene. Each line of data originates from a $3\times3$ square neighborhood of pixels completely contained within in the sub-area. Each line consists of the pixel values in the four spectral bands of each of the 9 pixels in such $3\times3$ neighborhood and a number indicating the corresponding label.
\item Event: This dataset consists of 8 sports event categories whose sizes vary from 137 to 250.
\item 13 Natural Scenes: This dataset contains 3859 images covering 13 classes of natural scenes. On average, each image has 250 $\times$ 300 pixels.
\item Caltech 101: This dataset includes 8677 images of objects covering 101 categories. There are about 40 to 800 pictures per category where each image is left-right aligned.
\item SImagenet: The whole Imagenet dataset contains about 1,260,000 images covering 1000 leaf synsets in a synset tree, in which each leaf synset represents a category of images. In our experiment, we construct a smaller dataset SImagenet by selecting 19,911 images from 12 synsets: manhole cover, daily, website, odometer, monarch butterfly, rapeseed, cliff dwelling, mountain, geyser, shoji, door, and villa.
\item Caltech 256: This original dataset consists of 30,607 images of objects from 256 categories and 1 background clutter. We discard the images in the clutter and selected the left images to form our dataset consisting of 29,780 images. Unlike Caltech 101, there more images in each category, and the images are not left-aligned in this dataset, therefore it is more difficult to identify.
\end{itemize}

\begin{table}
\caption{Statistics of Datasets}
\begin{tabular*}{8cm}{lllll}
\hline
Dataset &Size &Dimension &\#Classes &Data Type \\
\hline
wine & 178 & 13 & 3 & \\
seeds & 210 & 7 & 3 & \\
yeast & 1,484 & 17 & 10 & \\
isolet  & 1,559 & 617 & 26 &Audio \\
pendigits & 3,498 & 16 & 10 & \\
satimage & 6,435 & 36 & 7 & \\
\hline
Event & 1,579 & 10752 & 26 & Image \\
Scene 13 & 3,859 & 10752 & 13 & Image \\
Caltech 101 & 8,677 & 10752 & 101 & Image \\
SImagenet & 19,911 & 10752 & 12 & Image \\
Caltech 256 &29,780 & 10752 & 256 & Image \\

\hline
\end{tabular*}
\label{StatisticOfDatasets}
\end{table}

\subsubsection{Compared Method}
We compare our proposed REK1 and REK2 with Euler $k$-means \cite{EulerKmeans1}. Following \cite{EulerKmeans1}, we choose the parameter $\alpha$ from $\{1e-4, 0.001, 0.005, 0.01, 0.05, 0.1 : 0.1 : 2,  5, 10, 50, 100:1e2:900,1e3:1e3:1e4\}$ for both REK1 and REK2 over all datasets. For EulerK, on the one hand, we utilize the setting of $\alpha$ described in \cite{EulerKmeans1} over the high-dimensional datasets. On the other hand, we choose the parameter $\alpha$ from $\{1e-4, 0.001, 0.005, 0.01, 0.05, 0.1 : 0.1 : 2,  5, 10, 50, 100:1e2:900,1e3:1e3:1e4\}$ for low-dimensional datasets.
\subsubsection{Evaluation Metric}
Similar with EulerK, we utilize two widely used metrics including the clustering accuracy (ACC) \cite{AccNmi1}\cite{kernel kmeans 4} and Normalized Mutual Information (NMI) \cite{kernel kmeans 2}\cite{Nmi1}defined as follows to evaluate the clustering performance
\begin{equation}
\text{ACC} = \frac{1}{n}\sum_{i=1}^{n}\delta(p_{i}-\text{map}(q_{i})),
\end{equation}
where $p_{i}$ and $q_{i}$ are predicted label and ground truth label for $i$-th data point, respectively. Besides, $\delta(x,y)$ is an indicator function where $\delta(x,y)=1$ if $x=y$, $\delta(x,y)=0$ otherwise and map$(\cdot)$ permutes predicted labels to maximally match with the ground truth labels.
\begin{equation}
\text{NMI}(X;Y) = 2\frac{I(X;Y)}{H(X)+H(Y)},
\end{equation}
where $H(X)=-\sum_{i=1}^{n}p(x_i)logp(x_i)$, $I(X;Y)=\sum_{x}\sum_{y}p(x,y)log\frac{p(x,y)}{p(x)p(y)}$.
The values of both ACC and NMI lie in the range of $[0,1]$. The values 1 of both ACC and NMI indicates the perfect match with true distribution whereas 0 indicates perfect mismatch.
\subsection{Experiment Result}

\subsubsection{Deviation Degree}
We compute the \emph{Deviation Degree} $\kappa$ of centroids obtained by EulerK on all real datasets and show them in Table \ref{DeviationDegreeRealData}. Table \ref{DeviationDegreeRealData} illustrates that EulerK indeed acquires outlier-like centroids which deviate from the support domain of mapped data, thus they are not quite reasonable to reveal the true distribution. Moreover, the value of $\kappa$ is much larger over several low-dimensional datasets than that on high-dimensional ones, i.e., the centroids obtained by EulerK over the much high dimensional datasets are very closed to the unit hyper-sphere surface, we could conclude that it is the high dimension not the EulerK algorithm itself that brings the big performance improvement over these high-dimensional datasets.
\begin{table}[]
\centering
\caption{\emph{Deviation Degree} $\kappa$ of centroids obtained by EulerK on real datasets}
\label{DeviationDegreeRealData}
\begin{tabular}{c|c|c}
\hline
high/low dimension              & data set & deviation degree \\
\hline
\multirow{8}{*}{low dimension}  & wine         & 0.4209                    \\ \cline{2-3} % alpha = 300
                                & seeds        & 0.00070954                \\ \cline{2-3}
                                & yeast        & 0.0052002                  \\ \cline{2-3} % alpha = 1.4
                                & isolet       & 0.0091804                 \\ \cline{2-3}
                                & pendigits    & 0.19997                    \\ \cline{2-3}
                                & satimage     & 0.33827                   \\ \cline{2-3} % alpha = 100;
\hline
\multirow{5}{*}{high dimension} & Event        & 0.044                \\ \cline{2-3}
                                & Scene 13     & 0.077278                \\ \cline{2-3}
                                & Caltech 101  & 0.082337                \\ \cline{2-3}
                                & SImagenet    & 0.060476                \\ \cline{2-3}
                                & Caltech 256  & 0.0887                \\ \cline{2-3}
\hline
\end{tabular}
\end{table}

\begin{table*}
\footnotesize
\centering
\caption{NMI (\%) of Different Methods on low-dimensional datasets}
    \label{tab:NMI_Low}
    \setlength{\tabcolsep}{4mm}{
    \begin{tabular}{|c|c|c|c|c|c|c|c|c|c|c|c|c|}
    \hline
    \multirow{2}{*}{method}&
    \multicolumn{2}{c|}{wine}&\multicolumn{2}{c|}{seeds}&\multicolumn{2}{c|}{yeast}&\multicolumn{2}{c|}{isolet}&\multicolumn{2}{c|}{pendigits}&\multicolumn{2}{c|}{satimage}\cr\cline{2-13}
    &mean&std&mean&std&mean&std&mean&std&mean&std&mean&std\cr\hline
    EulerK       &48.06 &0.85 &49.37 &1.58  &67.95 &0.03  &53.24 &0.01 &68.25 &0.04 &20.87 &0.02  \cr\hline
    REK1         &49.35 &0.77 &54.69 &1.80  &68.07 &0.07  &52.36 &0.00 &67.00 &0.02 &22.30 &0.03 \cr\hline
    REK2         &48.65 &0.49 &51.23 &2.15  &68.11 &0.02  &53.97 &0.00 &67.18 &0.02 &21.79 &0.02  \cr\hline
    %            &0.4209& -   &0.00070954 &- &0.0052002 & - &0.0091804 &- &0.19997 &- &0.33827 &- \cr \hline
    %\hline
    \end{tabular}}
\end{table*}

\begin{table*}
\footnotesize
\centering
\caption{ACC (\%) of Different Methods on low-dimensional datasets}
    \label{tab:ACC_Low}
    \setlength{\tabcolsep}{4mm}{
    \begin{tabular}{|c|c|c|c|c|c|c|c|c|c|c|c|c|}
    \hline
    \multirow{2}{*}{method}&
    \multicolumn{2}{c|}{wine}&\multicolumn{2}{c|}{seeds}&\multicolumn{2}{c|}{yeast}&\multicolumn{2}{c|}{isolet}&\multicolumn{2}{c|}{pendigits}&\multicolumn{2}{c|}{satimage}\cr\cline{2-13}
    &mean&std&mean&std&mean&std&mean&std&mean&std&mean&std\cr\hline
    EulerK       &78.26 &0.62 &75.52 &1.23  &32.06 &0.04 &47.97 &0.10 &65.03 &0.13 &59.67 &0.15  \cr\hline
    REK1         &78.54 &0.45 &80.29 &1.24  &33.90 &0.07 &48.65 &0.19 &67.26 &0.11 &58.55 &0.10  \cr\hline
    REK2         &78.76 &1.02 &76.76 &1.71  &32.39 &0.06 &47.72 &0.09 &68.86 &0.15 &58.18 &0.11  \cr\hline

    %\hline
    \end{tabular}}
\end{table*}
%%%%%%%%%%%%%%%%%%%%%%% Table
\begin{table*}
\footnotesize
\centering
\caption{NMI (\%) of Different Methods on high-dimensional datasets}
    \label{tab:NMI_High}
    \setlength{\tabcolsep}{4mm}{
    \begin{tabular}{|c|c|c|c|c|c|c|c|c|c|c|}
    \hline
    \multirow{2}{*}{method}&
    \multicolumn{2}{c|}{Event}&\multicolumn{2}{c|}{Scene13}&\multicolumn{2}{c|}{Caltech 101}&\multicolumn{2}{c|}{SImagenet}&\multicolumn{2}{c|}{Caltech 256}\cr\cline{2-11}
    &mean&std&mean&std&mean&std&mean&std&mean&std\cr\hline
    EulerK      &38.31 &0.02 &56.85 &0.02 &48.52 &0.00 &50.81  &0.01 &33.18 &0.00  \cr\hline
    REK1        &37.31 &0.05 &54.53 &0.05 &47.40 &0.00 &47.48  &0.01 &31.35 &0.00  \cr\hline
    REK2        &36.90 &0.03 &54.12 &0.04 &47.48 &0.00 &51.43  &0.01 &32.63 &0.00  \cr\hline
    %            &0.044 &-    &0.077278&-  &0.082337&-  &0.060476&-   &0.0887&-     \cr \hline
    %\hline
    \end{tabular}}
\end{table*}

\begin{table*}
\footnotesize
\centering
\caption{ACC (\%) of Different Methods on high-dimensional datasets}
    \label{tab:ACC_High}
    \setlength{\tabcolsep}{4mm}{
    \begin{tabular}{|c|c|c|c|c|c|c|c|c|c|c|}
    \hline
    \multirow{2}{*}{method}&
    \multicolumn{2}{c|}{Event}&\multicolumn{2}{c|}{Scene13}&\multicolumn{2}{c|}{Caltech 101}&\multicolumn{2}{c|}{SImagenet}&\multicolumn{2}{c|}{Caltech 256}\cr\cline{2-11}
    &mean&std&mean&std&mean&std&mean&std&mean&std\cr\hline
    EulerK      &52.09 &0.24 &54.88 &0.12 &27.50 &0.03 &54.36  &0.07 &10.81 &0.00  \cr\hline
    REK1        &50.75 &0.28 &53.48 &0.15 &26.43 &0.02 &50.87  &0.02 &10.94 &0.00  \cr\hline
    REK2        &50.42 &0.17 &53.63 &0.20 &26.56 &0.03 &55.71  &0.08 &10.92 &0.00  \cr\hline
    %\hline
    \end{tabular}}
\end{table*}

\begin{figure}
 \centering
 \includegraphics[width=7.5cm]{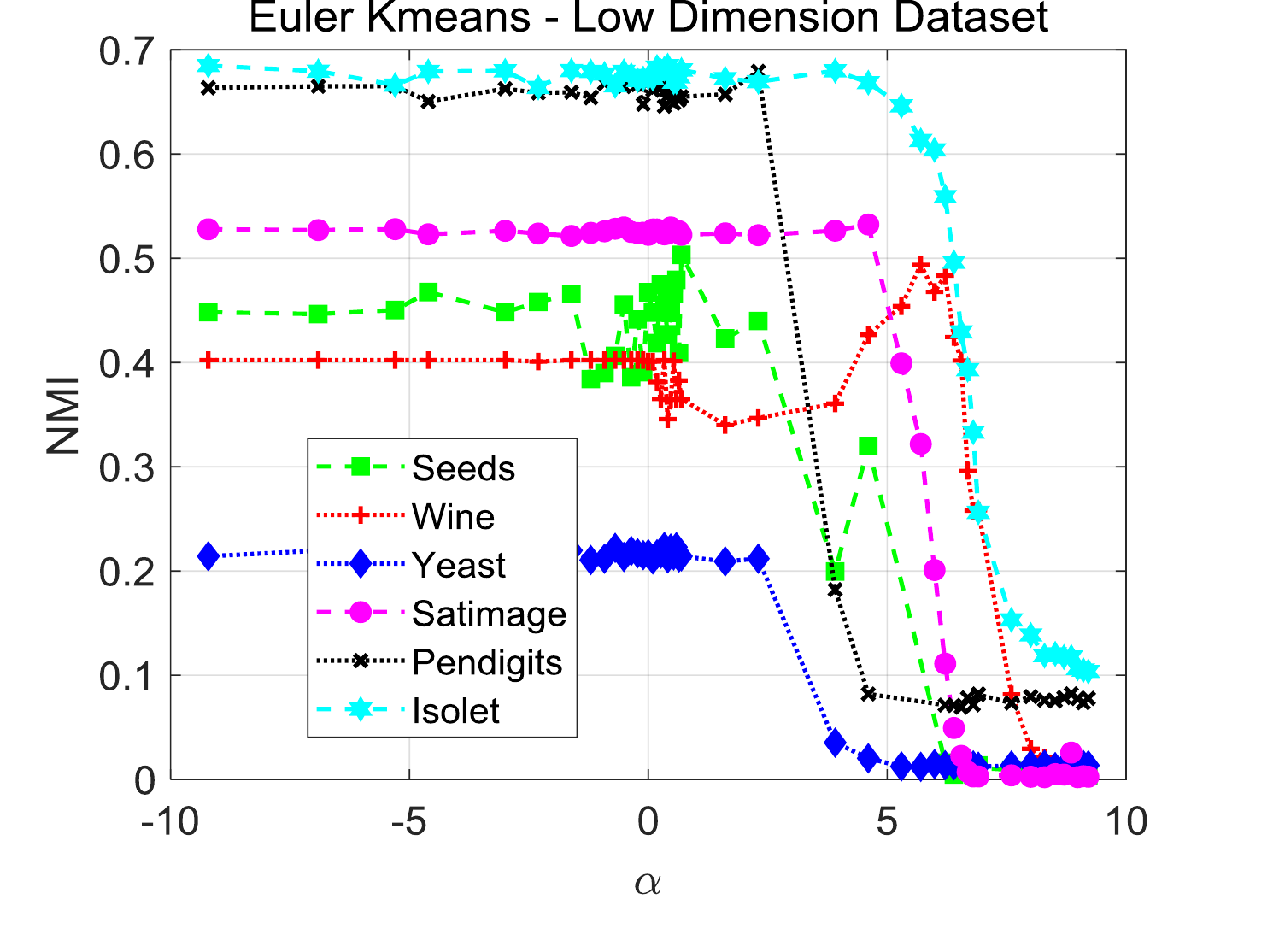}
 \caption{Clustering performance of EulerK on the six low-dimensional datasets with respect to $\alpha$.}
 \label{alphaEKLow}
 \end{figure}

 \begin{figure}
 \centering
 \includegraphics[width=7.5cm]{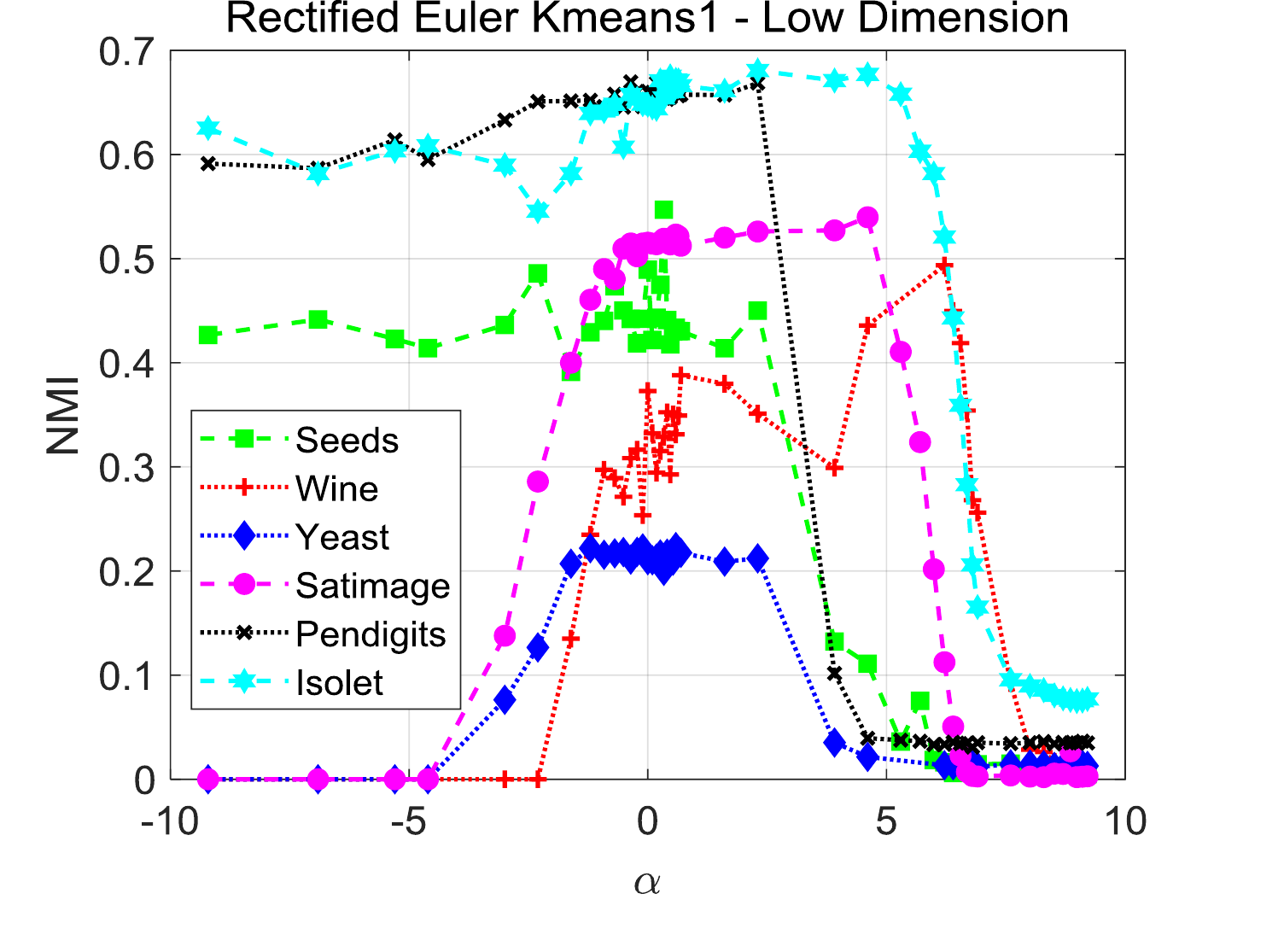}
 \caption{Clustering performance of REK1 on the six low-dimensional datasets with respect to $\alpha$.}
 \label{alphaREK1Low}
 \end{figure}
 \begin{figure}
 \centering
 \includegraphics[width=7.5cm]{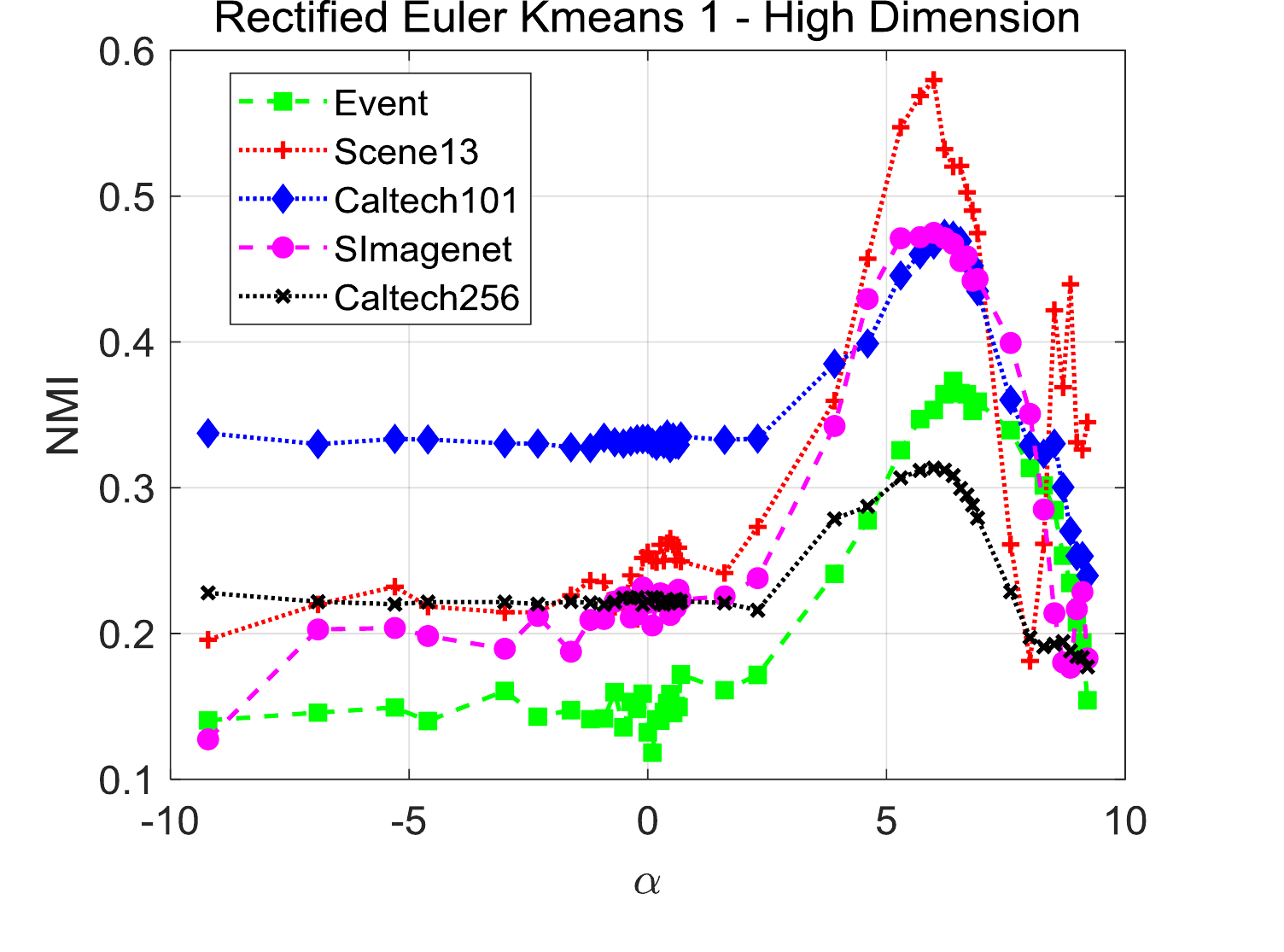}
 \caption{Clustering performance of REK1 on the five high-dimensional datasets with respect to $\alpha$.}
 \label{alphaREK1High}
 \end{figure}

 \begin{figure}
 \centering
 \includegraphics[width=7.5cm]{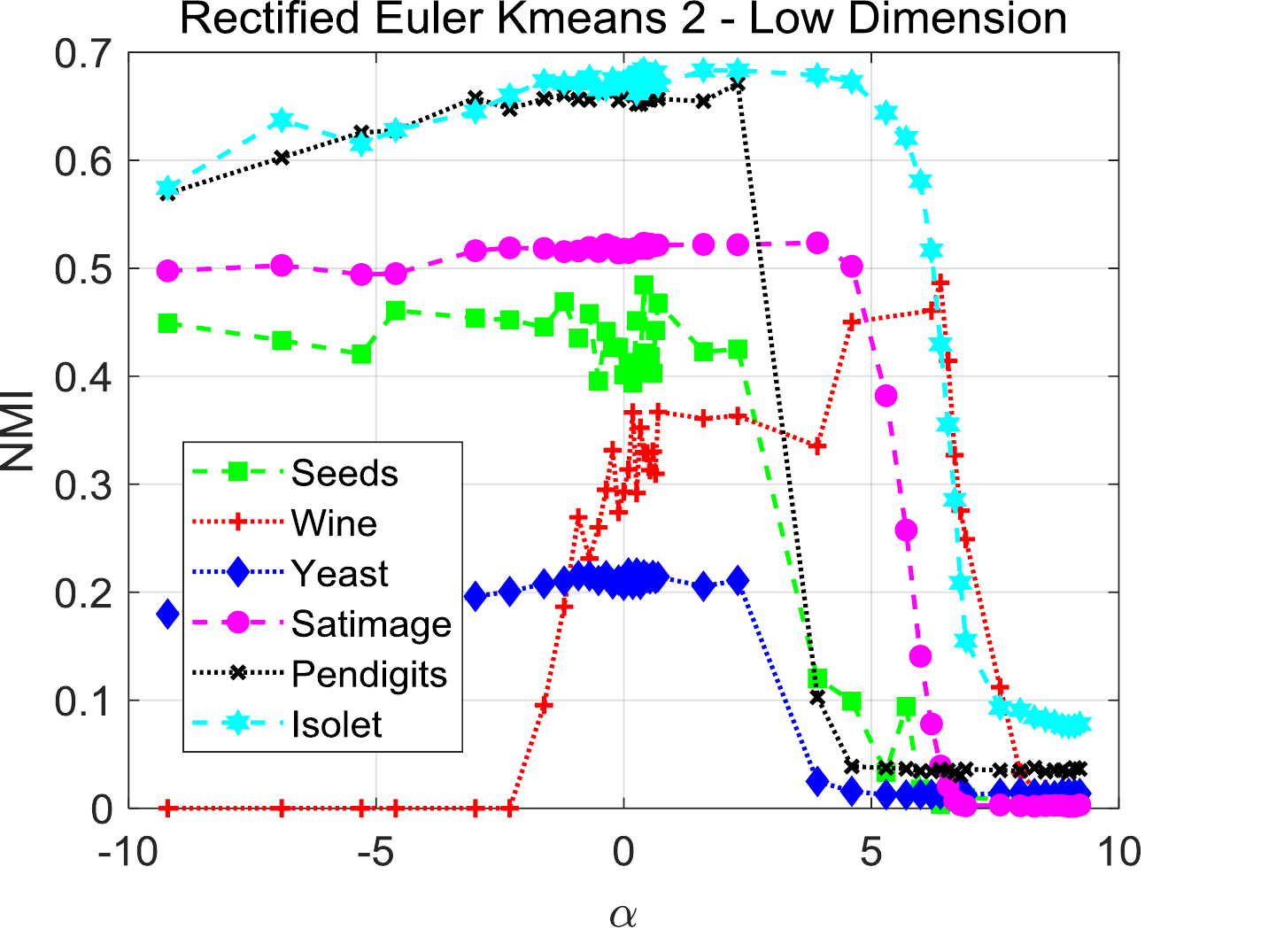}
 \caption{Clustering performance of REK2 on six low-dimensional datasets with different values of $\alpha$.}
 \label{alphaREK2Low}
 \end{figure}

 \begin{figure}
 \centering
 \includegraphics[width=7.5cm]{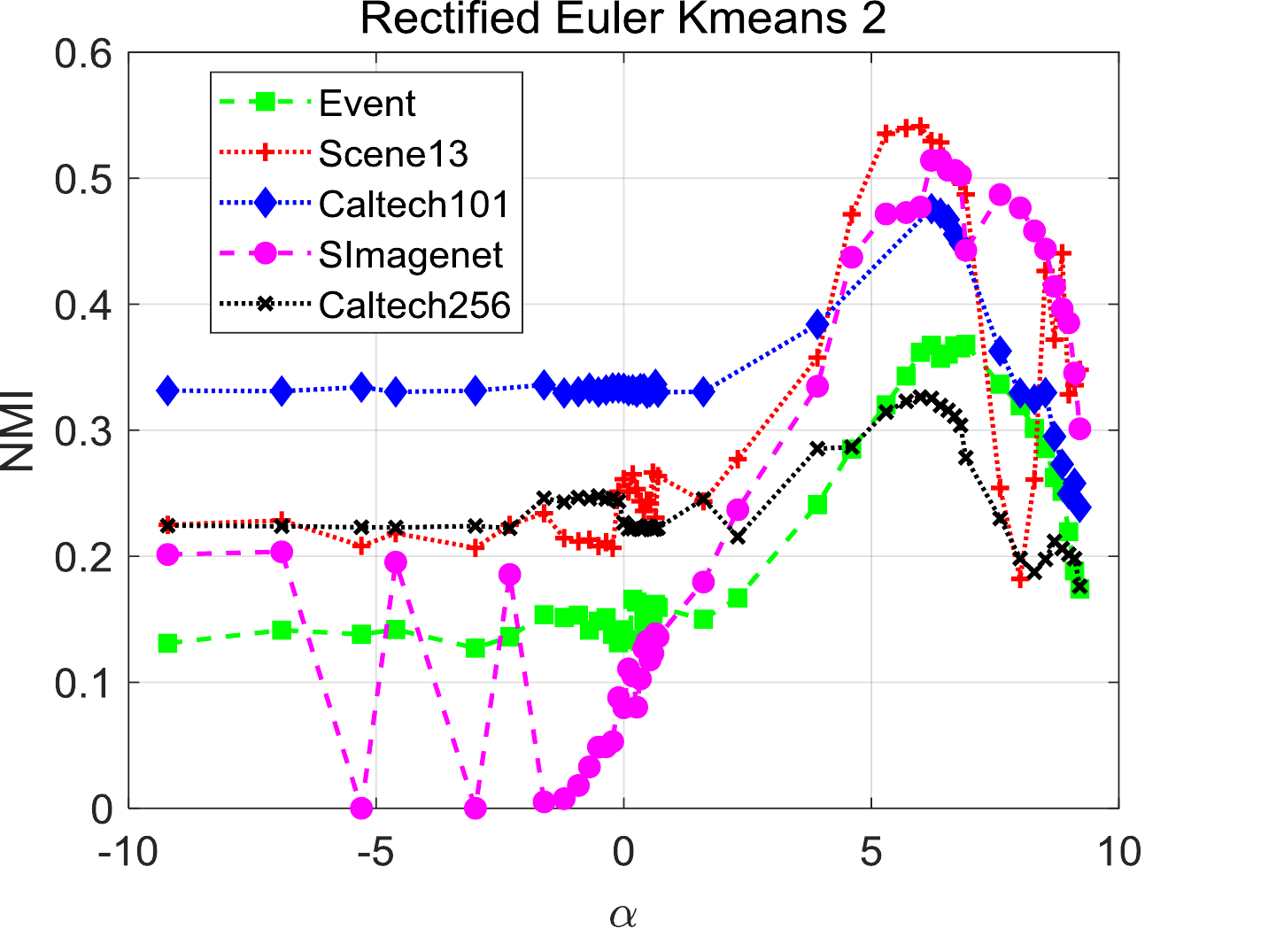}
 \caption{Clustering performance of REK2 on five high-dimensional datasets with different values of $\alpha$.}
 \label{alphaREK2High}
 \end{figure}
\subsubsection{Performance Evaluation}
Table \ref{tab:NMI_Low} and \ref{tab:ACC_Low} list the average ACC and NMI on the six low-dimensional datasets acquired by EulerK, REK1 and REK2. As Table \ref{tab:NMI_Low} and \ref{tab:ACC_Low} show, on seeds, REK1 performs better than EulerK with performance improvement of 5.32\% over NMI and 4.77\% over ACC, and REK2 performs better than EulerK with performance improvement of 1.86\% over NMI and 1.24\% over ACC. REK1 outpforms EulerK with 1.29\% over NMI while REK2 performs comparably with EulerK on wine. Moreover, both REK1 and REK2 perform comparably with EulerK on pendigits, satimage and isolet. Note similar with EulerK, we use only the groundtruth $K$ centroids to reflect the distribution of mapped data. Actually, it is unfair for our nonlinear models REK1 and REK2 to utilize only $K$ centroids to represent such manifold-like data points. That is, we should choose more centroids to represent the distribution of manifold-like dataset. Although so, our REK1 and REK2 still outperform EulerK on the low-dimensional datasets in summary. This further enhance the rationality of questioning the werid phenomenon occurring in EulerK and validate the effectiveness of both REK1 and REK2.

\par{Table \ref{tab:NMI_High} and \ref{tab:ACC_High} list the average ACC and NMI on the five high-dimensional datasets acquired by the compared approaches. For high-dimensional datasets with more than 10,000 features, as Table \ref{tab:NMI_High} and \ref{tab:ACC_High} show, both REK1 and REK2 achieve good performance comparable to EulerK, this illustrates the effectiveness of our proposed REK1 and REK2. On the one hand, from Table \ref{DeviationDegreeRealData}, we can see that EulerK obtains the centroids which approximately reside on the mapped space on these high-dimensional datasets , this in fact contributes to "curse of dimension" instead of the EulerK algorithm itself. That is EulerK actually obtains the centroids belonging to the support domain of mapped data on such very high-dimensional datasets, therefore our REK1 and REK2 perform comparably to EulerK on these high-dimensional dataset, further validating the rationality of REK1 and REK2. On the other hand, although EulerK performs comparably to REK1 and REK2, we choose the number of clusters as \cite{EulerKmeans2}, such a selection for the number of clusters is actually unfair for our methods. Therefore, it is necessary to determine more centroids to represent the distribution of manifold-like data points and design a new criterion to evaluate clustering performance on such categorical datasets.}
% 高维-维数灾难，而不是算法本身带来的性能提升？ 虽然性能相当，但是需要注意的是我们选取的聚类个数是2，这本身对于我们的算法是不公平的；对于这种流形上的数据，我们应当设计一种新的聚类算法性能衡量准则；

% 超参数记录：seeds-REK1-1.4 wine-REK1-500-REK2-1.6; yeast-REK1-1.8-REK2-1.3; satimage-REK1-100-REK2-50

 \begin{figure}
\centering
\subfigure[The objective function value of REK1 versus iteration on Scene13 dataset.]{%
\includegraphics[width=5.55cm]{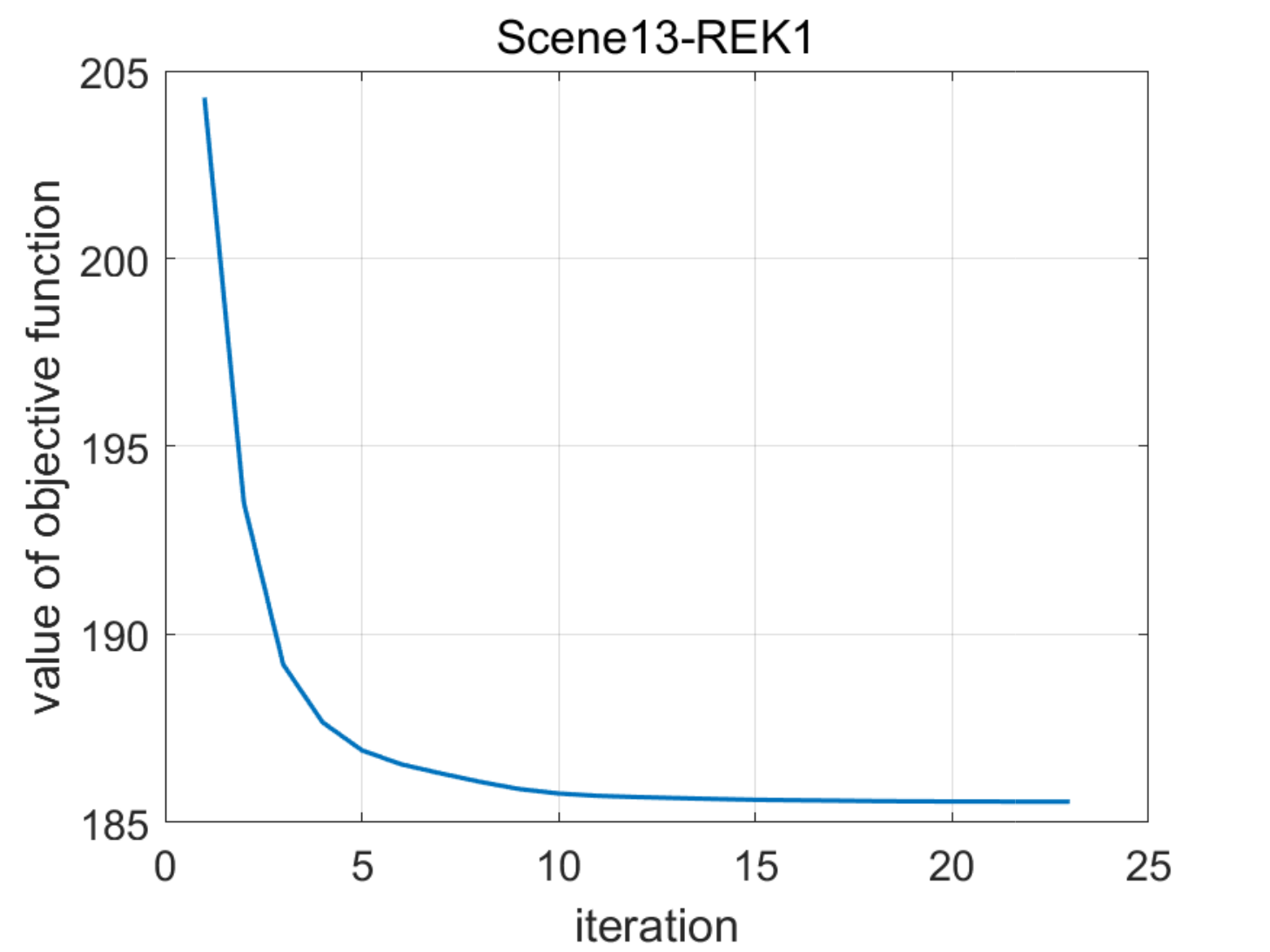}
\label{ConvergenceREK1_Scene13}}
\quad
\subfigure[The objective function value of REK2 versus iteration on Scene13 dataset.]{%
\includegraphics[width=5.55cm]{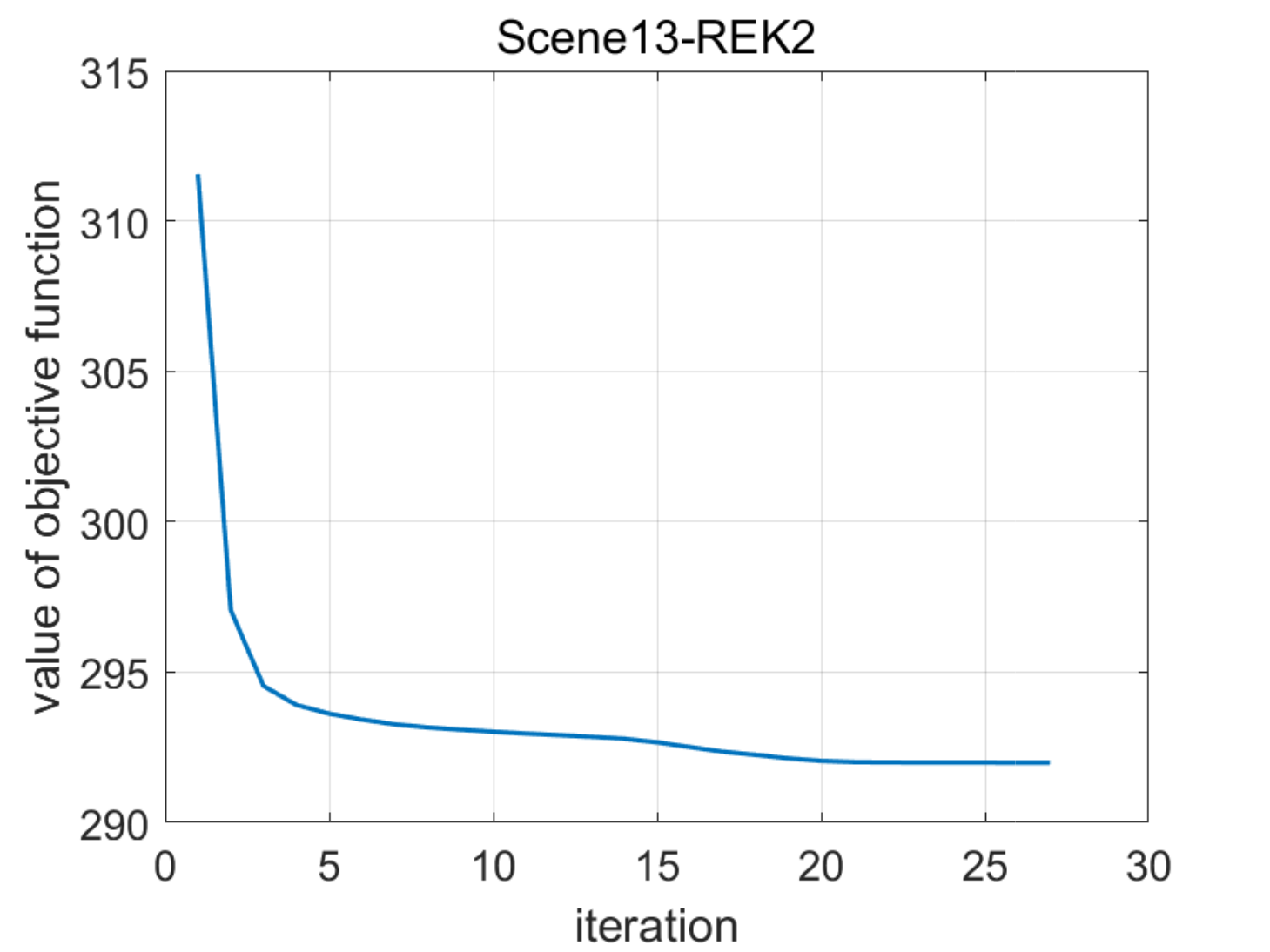}
\label{ConvergenceREK2_Scene13}}
\caption{The objective function value of REK1 and REK2 versus iteration Scene13 dataset.}
\label{ConvergenceScene13}
\end{figure}

\begin{figure}
\centering
\subfigure[The objective function value of REK1 versus iteration on Caltech256 dataset.]{%
\includegraphics[width=5.55cm]{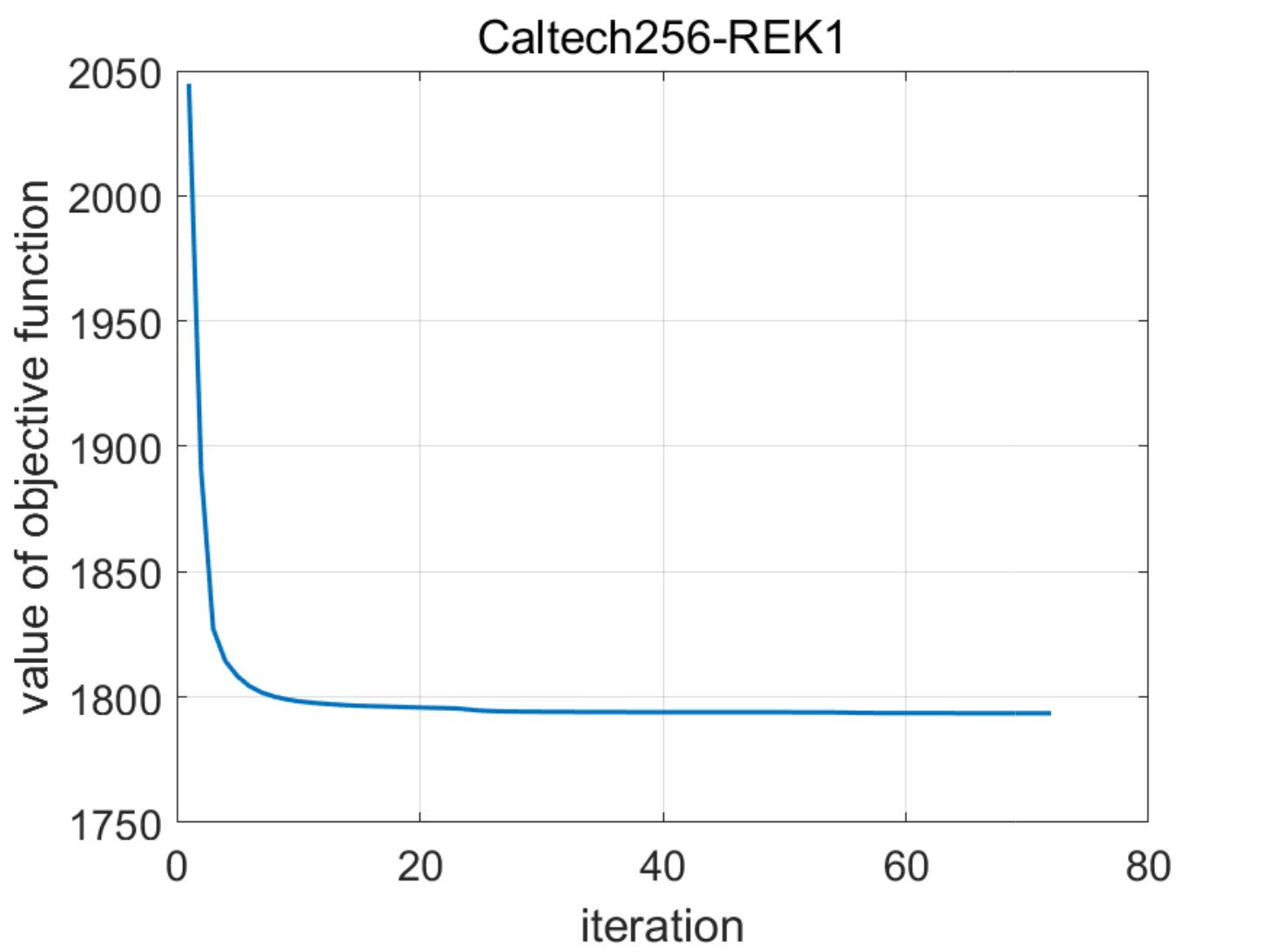}
\label{ConvergenceREK1_Caltech256}}
\quad
\subfigure[The objective function value of REK2 versus iteration on Caltech256 dataset.]{%
\includegraphics[width=5.55cm]{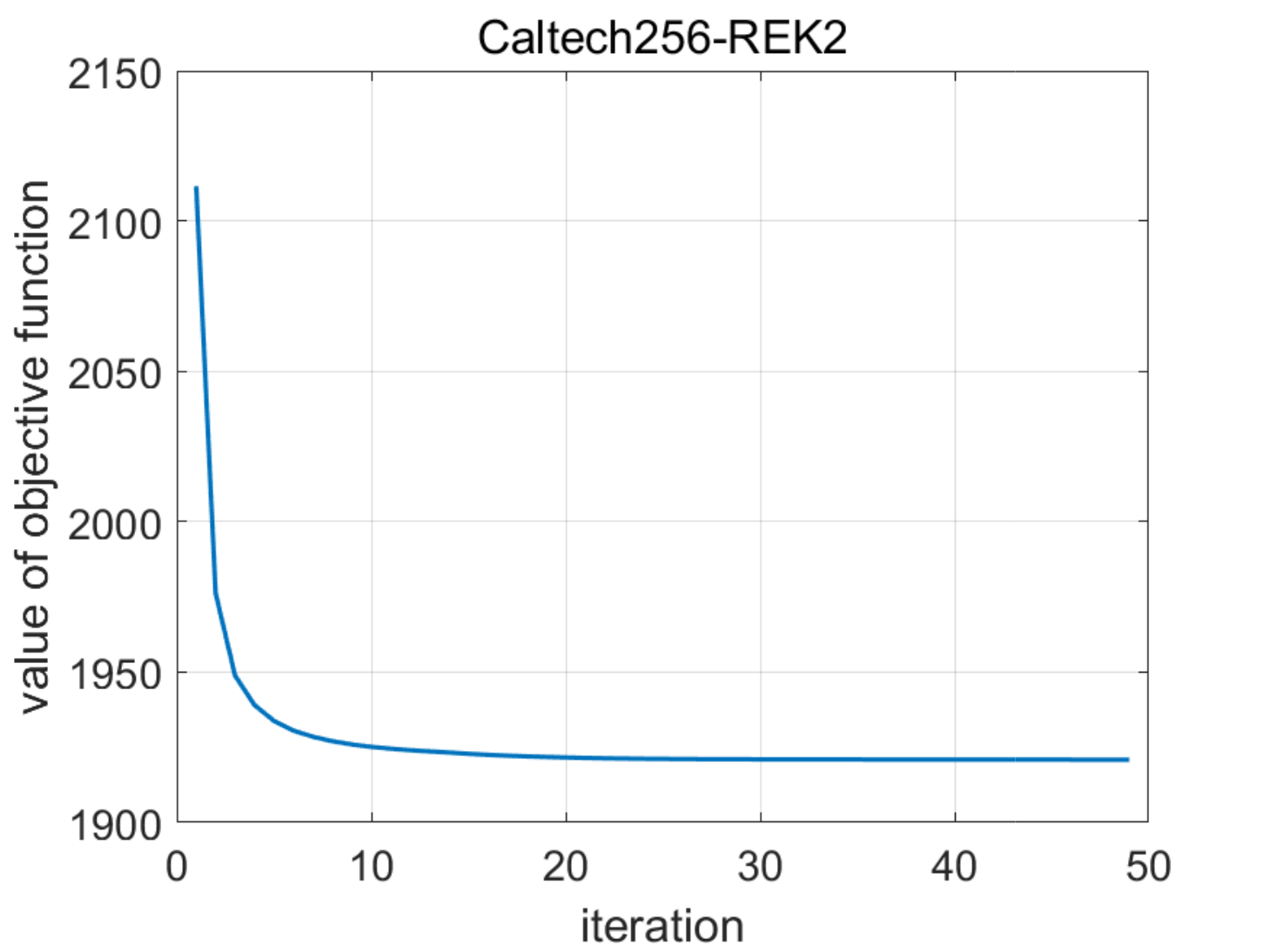}
\label{ConvergenceREK2_Caltech256}}
\caption{The objective function value of REK1 and REK2 versus iteration Caltech256 dataset.}
\label{ConvergenceCaltech256}
\end{figure}

\subsection{Parameters}
In this section, we rectify the declaration about the effect of $\alpha$ on the performance of EulerK in \cite{EulerKmeans1}\cite{EulerKmeans2} and analyse the performance of the proposed REK1 and REK2 against different $\alpha$ in Euler kernel.
\par{Figure \ref{alphaEKLow} shows average NMI obtained by EulerK on the six low-dimensional datasets with respect to $\alpha$. Compared this figure with Figure 8 about five high-dimensional datasets in \cite{EulerKmeans1}\cite{EulerKmeans2}, we can find that the influence of $\alpha$ on the performance of EulerK is closely related to the dimension of data point. From Figure \ref{alphaEKLow}, we can see that EulerK performs well on the small values of $\alpha$ for the low-dimensional datasets. Different from the low-dimensional datasets, EulerK performs much better on the large values of $\alpha$ over the high-dimensional datasets. In summary, the insensitive range to $\alpha$ of EulerK over low-dimensional datasets is exactly opposite of that over high-dimensional datasets.}
\par{Figure \ref{alphaREK1Low} and \ref{alphaREK1High} illustrate the clustering performance of REK1 on low-dimensional and high-dimensional datasets with respect to the value of $\alpha$, respectively. From the two figures, we can find that $\alpha$ also has different from effect on the performance of REK1 on the different dimensional datasets. Similar with EulerK, small $\alpha$ brings large performance improvement for REK1 on low-dimensional datasets whereas large $\alpha$ brings large performance improvement for REK1 on the high-dimensional datasets. Figure \ref{alphaREK2Low} and \ref{alphaREK2High} show the clustering performance of REK2 on low-dimensional and high-dimensional datasets over the value of $\alpha$, respectively. From them, we can see that the similar declaration about $\alpha$ on REK2 with REK1 is also acquired.}

\subsection{Convergence Analysis}
In this section, we plot the objective function values of (\ref{REK1}) and (\ref{REK2}) with respect to number of iterations on Scene13 and Caltech256 in Figure \ref{ConvergenceScene13} and \ref{ConvergenceCaltech256}, respectively.  From the two figures, we can observe that the two objective functions monotonously and fast declines and converge after several iterations, illustrating the fast convergence of both REK1 and REK2.

\section{Conclusion}
In this paper, we focus on a simple but elegant kernel clustering method, i.e., Euler $k$-means (EulerK). Although EulerK is simple, feasible to large-scale clustering problem, robust to noise and outliers, it acquires centroids deviating from the mapped feature space, which in strict distributional sense, actually are outliers. This weird phenomenon also occurs in some popular generic kernel clustering methods but is seldom concerned so far. Motivated by this phenomenon, in this paper, propose two rectified methods, i.e. Rectified Euler Kmeans 1 (REK1) which rectify EulerK by adding certain constraint on these centroids in the complex space and Rectified Euler Kmeans 2 (REK2) which treats each centroid as a mapped image of a data point or pre-image in the original space and optimizes these pre-images in Euler kernel induced metric space. The two proposed methods not only inherit the merits of EulerK but also acquire the true centroids which really reside on the mapped space to better characterize the structure of mapped data. Moreover, these two methods can be  methodologically extended to straightforwardly deal with the problems of such a category in generic kernel clustering approaches. Finally, the experimental results on both synthetic and commonly used real datasets validate the rationality and effectiveness of our proposed REK1 and REK2. In the future, we will try to propose a new specific measure to evaluate the clustering performance of kernel methods with such phenomenon.

\section{Acknowledgments}
This work is supported by the Key Program of National Natural and Science Foundation of China (NSFC) under Grant No. 61732006.

\ifCLASSOPTIONcompsoc

%\else
%\fi

\ifCLASSOPTIONcaptionsoff
  \newpage
\fi

\begin{IEEEbiography}[{\includegraphics[width=1in,height=1.25in,clip,keepaspectratio]{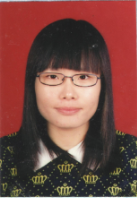}}]{Yunxia Lin}
received the B.S. degree in electronic information engineering from Qilu University of Technology Jinan, China, in 2013. In 2017, she completed her M.S. degree in communication and information system at Lanzhou University, Lanzhou, China. Currently, she is pursuing the Ph.D degree with the College of Computer Science \& Technology, Nanjing University of Aeronautics and Astronautics, Nanjing, China. Her research interests include pattern recognition and machine learning.
\end{IEEEbiography}
\begin{IEEEbiography}[{\includegraphics[width=1in,height=1.25in,clip,keepaspectratio]{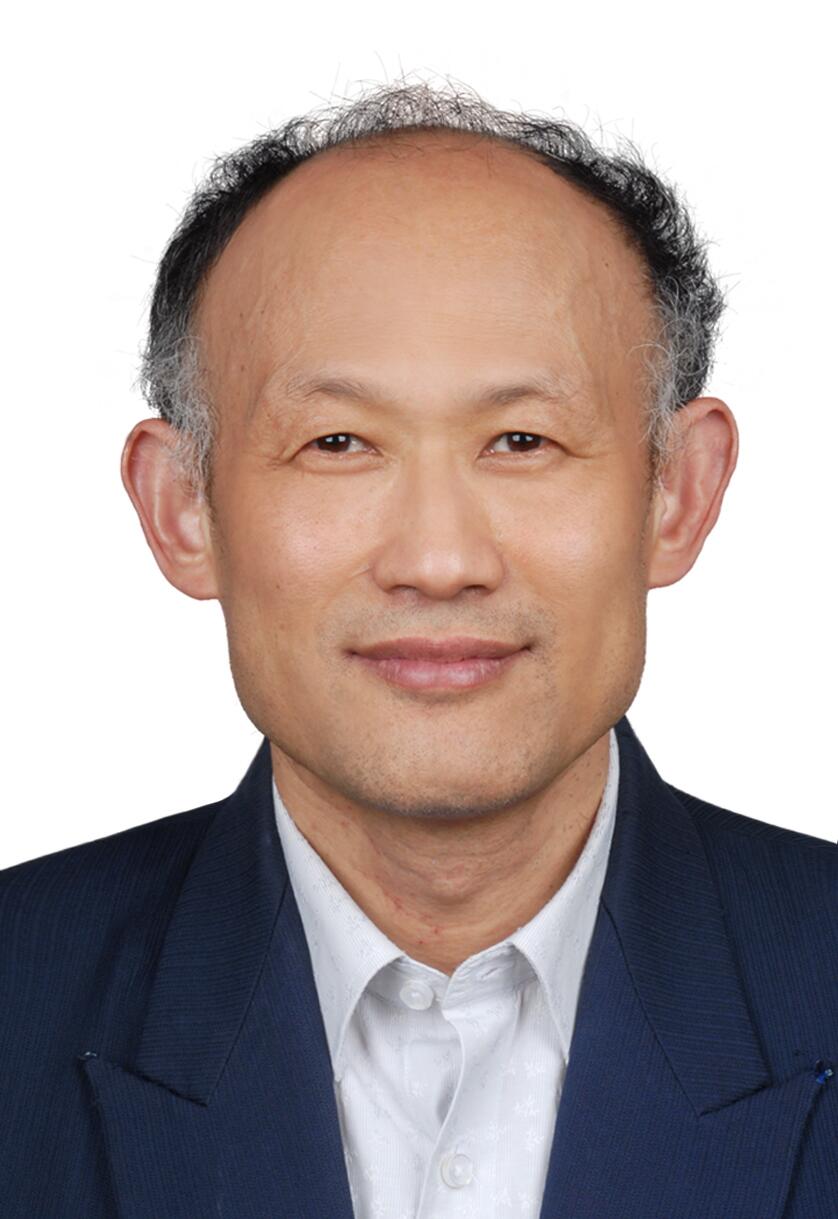}}]{Songcan Chen}
received his B.S. degree in mathematics from Hangzhou University (now merged into Zhejiang University) in 1983. In 1985, he completed his M.S. degree in computer applications at Shanghai Jiaotong University and then worked at Nanjing University of Aeronautics and Astronautics (NUAA) in January 1986, where he received a Ph.D. degree in communication and information systems in 1997.
\par{Since 1998, as a full-time professor, he has been with the College of Computer Science \& Technology at NUAA. His research interests include pattern recognition, machine learning and neural computing. He has published over 100 top-tier journals, such as TPAMI, TKDE, TNNLS, TIP, IEEE Transactions on Information Forensics \& Security, IEEE Transactions on Systems, Man \& Cybernetics-Part B, IEEE Transactions on Wireless Communication and conference papers, such as ICML, CVPR, IJCAI, AAAI, ICDM and so on. He is also an IAPR Fellow.}
\end{IEEEbiography}


\begin{thebibliography}{1}
%% clustering
\bibitem{data mining1}
F. Hoppner, F. Klawonn, R. Kruse, and T. Runkler, “Fuzzy cluster analysis: methods for classification, data analysis and image recognition,” \emph{John Wiley \& Sons}, pp. 41–53, 1999.
\bibitem{data mining2}
D. Huang, C.D. Wang, J.S. Wu, J.H. Lai, and C.K. Kwo, “Ultra-scalable spectral clustering and ensemble clustering,” \emph{IEEE Trans. Knowl. Data Eng.},  vol. 32, no. 6, pp. 1212–1226, Jun. 2020.
\bibitem{clusteringApp_machine learning1}
R. Xu and D. I. Wunsch, “Survey of clustering algorithms,” \emph{IEEE Trans. Neural Netw.}, vol. 16, no. 3, pp. 645–678, May. 2005.
\bibitem{clusteringApp_machine learning2}
A. K. Jain, “Data clustering: 50 years beyond k-means,” \emph{Pattern Recognit. Lett.}, vol. 31, no. 8, pp. 651–666, 2010.
\bibitem{image processing}
C. G. Li, C. You, and R. Vidal, “Structured sparse subspace clustering: a joint affinity learning and subspace clustering framework,” \emph{IEEE Trans. Image Process.}, vol. 26, no. 6, pp. 2988–3001, Jun. 2017.
\bibitem{what is clustering}
D. Cheng, Q. Zhu, J. Huang, Q. Wu, and L. Yang, “Clustering with local density peaks-based minimum spanning tree,” \emph{IEEE Trans. Knowl. Data Eng.}, vol. 33, no. 2, pp. 374-387, Feb. 2021.
\bibitem{clustering category 1}
P. Berkhin, “A survey of clustering data mining techniques,” in \emph{Grouping Multidimensional Data: Recent Advances in Clustering}, Eds. New York: Springer-Verlag, 2006, pp. 25–71.

%P. Berkhin, “A survey of clustering data mining techniques,” in \emph{Grouping Multidimensional Data: Recent Advances in Clustering}, J. Kogan, C. Nicholas, and M. Teboulle, Eds. New York: Springer-Verlag, 2006, pp. 25–71.
%P. Berkhin, “A Survey of Clustering Data Mining Techniques, Chapter: Grouping Multidimensional Data, pp 25-71.
\bibitem{clustering category 2}
A. A. Abbasi and M. Younis, “A survey on clustering algorithms for wireless sensor networks,” \emph{Computer Commun.}, vol. 30, pp. 2826–2841, 2007.
%Ameer Ahmed Abbasi et al, A survey on clustering algorithms for wireless sensor networks, Computer Communications, 30(14-15) :2826–2841,2007.
\bibitem{clustering category 3}
T. W. Liao, “Clustering of time series data—A survey,” \emph{Pattern Recognit.}, vol. 38, no. 11, pp. 1857–1874, 2005.
%T. Warren Liao, Clustering of time series data—a survey, Pattern Recognition, 38(11): 1857–1874,2005.
%% kmeans
\bibitem{kmeans}
J.A. Hartigan and M.A. Wong, “Algorithm AS 136: A K-means clustering algorithm,” \emph{J. R. Stat. Soc. Ser. C-Appl. Stat.}, vo.28, no.1, pp. 100-108, 1979.
\bibitem{kmeans-wide use}
A. K. Jain, “Data clustering: 50 years beyond k-means,” \emph{Pattern Recognit. Lett.}, vol. 31, no. 8, pp. 651–666, 2010.
\bibitem{kmeans-popular}
I. Dhillon, Y. Guan, and B. Kulis, “Kernel k-means, spectral clustering and normalized cuts,” in \emph{Proc. 10th ACM KDD Conference}, 2004, pp. 551–556.
\bibitem{kmeansCom}
K. Alsabti, S. Ranka, and V. Singh, “An efficient k-means clustering algorithm,” \emph{Proc. First Workshop High Performance Data Mining}, Mar. 1998.
\bibitem{KmeansAssume}
R. Zhang, A.I. Rudnicky, “A large scale clustering scheme for kernel $k$-Means,” in \emph{Proc. 16th Int’l Conf. Pattern Recognition}, 2002.
\bibitem{kernel kmeans 1}
M. Girolami, “Mercer kernel based clustering in feature space,” \emph{IEEE Trans. Neural Netw.,} vol. 13, no. 3, pp.780–784, May. 2002.
\bibitem{kernel kmeans 2}
S. Yu, L.-C. Tranchevent, X. Liu, W. Gl{\"a}nzel, J. A. K. Suykens,B. D. Moor, and Y. Moreau, “Optimized data fusion for kernel $k$-means clustering,” \emph{IEEE Trans. Pattern Anal. Mach. Intell.}, vol. 34, no. 5, pp. 1031–1039, May. 2012.
\bibitem{kernel kmeans 3}
M. G{\"o}nen and A. A. Margolin, “Localized data fusion for kernel k-means clustering with application to cancer biology,” in \emph{Proc. 27th Int. Conf. Neural Inf. Process. Syst.,} 2014, pp. 1305–1313.
\bibitem{kernel kmeans 4}
He. Li and H. Zhang, “Kernel K-means sampling for Nyström approximation,” \emph{IEEE Trans. Image Process.}, vol. 27, no.5, pp. 2108-2120, May. 2018.
%% kernel clustering
\bibitem{kernel clustering 1}
D. Marin, M. Tang, I. Ben Ayed, and Y. Boykov, “Kernel clustering: Density biases and solutions,” \emph{IEEE Trans. Pattern Anal. Mach. Intell.}, vol. 41, no. 1, pp. 136–147, Jan. 2019.
\bibitem{kernel clustering 2}
Z.W. Ren, Q.S. Sun, and D. Wei, “Multiple kernel clustering with kernel k-Means coupled graph tensor learning,” in \emph{Proc. of the AAAI Conf Artif Intell.}, 2021.
\bibitem{RBF kmeans}
R. Chitta, R. Jin, T. C. Havens, and A. K. Jain, “Approximate kernel k-means: solution to large scale kernel clustering,” in \emph{Proc. ACM SIGKDD Int. Conf. Knowl. Discovery Data Mining}, 2011, pp. 895–903.
\bibitem{SPM kmeans}
S. Lazebnik, C. Schmid, and J. Ponce, “Beyond bags of features: spatial pyramid matching for recognizing natural scene categories,” in \emph{Proc. IEEE Conf. Comput. Vis. Pattern Recognit.}, 2006, pp. 2169–2178.
\bibitem{kernel trick}
B. Sch{\"o}lkopf, “The kernel trick for distances,” in \emph{Proc. Int. Conf. Neural Inf. Process. Syst.}, 2000, pp. 301–307.
\bibitem{App1}
N. Pham and R. Pagh, “Fast and scalable polynomial kernels via explicit feature maps,” in \emph{Proc. ACM SIGKDD Int. Conf. Knowl. Discovery Data Mining}, 2013, pp. 239–247.
\bibitem{App2}
A. Rahimi and B. Recht, “Random features for large-scale kernel machines,” in \emph{Proc. Int. Conf. Neural Inf. Process. Syst.}, 2007,
pp. 1177–1184.
\bibitem{App3}
R. Chitta, R. Jin, and A. K. Jain, “Efficient kernel clustering using random fourier features,” in \emph{Proc. IEEE Int. Conf. Data Mining}, 2012, pp. 161–170.
\bibitem{App4}
P. Kar and H. Karnick, “Random feature maps for dot product kernels,” in \emph{Proc. Int. Conf. Artif. Intell. Statist.}, 2012, pp. 583–591.
\bibitem{App5}
A. Vedaldi and A. Zisserman, “Efficient additive kernels via explicit feature maps,” \emph{IEEE Trans. Pattern Anal. Mach. Intell.}, vol. 34, no. 3, pp. 480–492, Mar. 2012.
\bibitem{EulerKmeans1}
J.S. Wu, W.S. Zheng, J.H. Lai, “Euler clustering,” in \emph{Proc. Int. Joint Conf. Artif. Intell.}, 2013, pp. 1792-1798.
%J.S. Wu, W.S. Zheng, J.H. Lai, “Euler clustering,” in \emph{Proc. of the Twenty-Third International Joint Conference on Artificial Intelligence}, 2013, pp. 1792-1798.
\bibitem{EulerKmeans2}
J.S. Wu, W.S. Zheng, C. Y.Suen, “Euler clustering on large-scale dataset,” \emph{IEEE Trans. Big Data}, vol. 4, no. 4, Dec. 2018, pp. 502-515, 2018.
\bibitem{OutlierCite}
V. Barnett and T. Lewis, “Outliers in statistical data,” John Wiley and Sons, NY 1994.

%%%%%%%%%%%% traditional kernel clustering methods

\bibitem{KernelFCM}
D. Q. Zhang and S. C. Chen, “A novel kernelized fuzzy c-means algorithm with application in medical image segmentation,” \emph{Artif. Intell. Med.}, vol. 32, pp. 37–52, 2004.
%D.Q. Zhang, S.C. Chen, A novel kernelized fuzzy c-means algorithm with application in medical image segmentation, Artificial intelligence in medicine 32 (1), 37-50.
\bibitem{KSOM1}
R. Inokuchi and S. Miyamoto, “LVQ clustering and SOM using a kernel function,” in \emph{Proc. of IEEE Int. Conf. Fuzzy Syst.}, vol. 3, pp. 1497–1500, 2004.
%R. Inokuchi and S. Miyamoto. LVQ clustering and SOM using a kernel function. In Proceedings of IEEE International Conference on Fuzzy Systems, volume 3, pages 1497–1500, 2004.
\bibitem{KSOM2}
D. MacDonald and C. Fyfe, “The kernel self-organising map,” in \emph{Proc. 4th Int. Conf. Knowl.-Based Intell. Eng. Syst. Allied Technol.}, 2000, vol. 1, pp. 317–320.
%D. Macdonald and C. Fyfe. The kernel self-organising map. In Fourth International Conference on Knowledge-Based Intelligent Engineering Systems and Allied Technologies, 2000, volume 1, pages 317–320, 2000.
\bibitem{KernelNeuralgas}
A.K. Qin and P.N. Suganthan, "Kernel neural gas algorithms with application to cluster analysis", in \emph{Proc. Int’l Conf. Pattern Recognition}, 2004.
%A. K. Qinand and P. N. Suganthan, Kernel neural gas algorithms with application to cluster analysis. ICPR, 04:617–620, 2004.
\bibitem{MultiKernelKmeans}
H. C. Huang, Y. Y. Chuang, and C. S. Chen, “Multiple kernel fuzzy clustering,” \emph{IEEE Trans. Fuzzy Syst.}, vol. 20, no. 1, pp. 120–134, Feb. 2012.
%Huang, H.; Chuang, Y.; and Chen, C. 2012. Multiple kernel fuzzy clustering. IEEE T. Fuzzy Systems 20(1):120–134.
\bibitem{KMD-FC}
S. Zeng, X. Wang, X. Duan, S. Zeng, Z. Xiao, D. Feng, “Kernelized mahalanobis distance for fuzzy clustering,” \emph{IEEE Trans. Fuzzy Syst.}, Jul. 2020.
%S. Zeng, X. Wang, X. Duan, S. Zeng, Z. Xiao, D. Feng, “Kernelized mahalanobis distance for fuzzy clustering,” IEEE Transactions on Fuzzy Systems, doi: 10.1109/TFUZZ.2020.3012765, Jul. 2020.
%%%%%%%%%%%%%%%%%% Related work

% Euler kernel
\bibitem{EulerKernel}
S. Liwicki, G. Tzimiropoulos, S. Zafeiriou, and M. Pantic, “Euler principal component analysis,” \emph{Int. J. Comput. Vis.}, vol. 101, no. 3, pp. 498–518, 2013.
\bibitem{Laplace kernel}
L.M. Yang, Z. Ren, Y.D. Wang, and H.W. Dong, “A robust regression framework with laplace kernel-induced loss,” \emph{Neural. Comput.}, vol. 29, no. 11, pp. 3014-3039, Oct. 2017.
%S. Liwicki, G. Tzimiropoulos, S. Zafeiriou, and M. Pantic, Euler principal component analysis. International Journal of Computer Vision, 2012.
%%%%%%%%%%%% Convergence analysis
\bibitem{ConvergenceKmeans}
S.Z. Selim and M.A. Ismail, “$k$-means-type algorithms: a generalized convergence theorem and characterization of local optimality,” \emph{IEEE Trans. Pattern Anal. Mach. Intell.}, vol. 6, pp. 81-87, 1984.
%SHOKRI Z. SELIM AND M. A. ISMAIL, K-Means-Type Algorithms: A Generalized Convergence Theorem and Characterization of Local Optimality, IEEE TRANSACTIONS ON PATTFRN ANALYSIS AND MACHINE INTELLIGENCE, VOL. PAMI-6, NO. 1, JANUARY 1984, pp. 81-87.

%%%%%%%%%%%%%%%%%% Dataset
\bibitem{Event}
L.J. Li and F.F Li, “What, where and who? classifying events by scene and object recognition,” in \emph{Proc. IEEE Int. Conf. Comput. Vis.}, 2007, pp. 1–8.
\bibitem{Scene 13}
F.F. L. and P. Perona, “A Bayesian hierarchical model for learning natural scene categories,” in \emph{Proc. IEEE Conf. Comput. Vis. Pattern Recognit.}, 2005, pp. 524–531.
\bibitem{Caltech 101}
F.F. Li, R. Fergus, and P. Perona, “One-shot learning of object categories,” \emph{IEEE Trans. Pattern Anal. Mach. Intell.}, vol. 28, no. 4, pp. 594–611, Apr. 2006.
\bibitem{SImagenet}
J. Deng, W. Dong, R. Socher, L.-J. Li, K. Li, and F.F. Li, “ImageNet: a large-scale hierarchical image database,” in \emph{Proc. IEEE Conf. Comput. Vis. Pattern Recognit.}, 2009, pp. 248–255.
\bibitem{Caltech 256}
G. Griffin, A. Holub, and P. Perona, “Caltech-256 object category dataset,” \emph{California Inst. Technol.}, Pasadena, CA, USA, Tech. Rep. 7694, 2007.

%%% ACC & NMI
\bibitem{AccNmi1}
S. Zhou, X.W. Liu, M.m. Li, E. Zhu, L. Liu, C.w. Zhang, J.p. Yin, “Multiple kernel clustering with neighbor-kernel subspace segmentation,” \emph{IEEE Trans. Neural Netw. Learn. Syst.}, vol. 31, no. 4, pp. 1351–1362, Apr. 2020.
%\bibitem{Acc1}
%L. He and H. Zhang, “Kernel K-means sampling for Nyström approximation,” \emph{IEEE Trans. Image Process.}, vol. 27, no. 5, pp. 2108–2120, May. 2018.
\bibitem{Nmi1}
G. Tzortzis and A. Likas, “The global kernel $k$-means algorithm for clustering in feature space,” \emph{IEEE Trans. Neural Netw.}, vol. 20, no. 7, pp. 1181–1194, Jul. 2009.



\end{thebibliography}
\end{document}